\documentclass[twoside]{article}

     \usepackage[accepted]{aistats2021}

\usepackage{savesym,multirow,url,xspace,graphicx,enumitem,color}

\usepackage[dvipsnames]{xcolor}
\usepackage{hyperref}
\hypersetup{
citecolor = Blue,    
colorlinks = true,
  linkcolor = Blue,
}

\usepackage[titletoc]{appendix}
\usepackage{subcaption}
\usepackage{dsfont}
\usepackage{multicol}

\usepackage{wrapfig}

\usepackage{etoolbox}
\usepackage[utf8]{inputenc} 
\usepackage[T1]{fontenc}    
\usepackage{hyperref}       
\usepackage{url}            
\usepackage{booktabs}       
\usepackage{amsfonts}       
\usepackage{nicefrac}       
\usepackage{microtype}      
\usepackage{tikz}
\usepackage{pgfplots}
\usepackage{verbatim}
\usepackage{amsmath,bm}
\usepackage{multirow}
\usepackage{mathtools}
\usepackage{comment}

\newcommand{\argmin}{\operatornamewithlimits{arg\,min}}

\newtoggle{showcomments}
\settoggle{showcomments}{true}

\iftoggle{showcomments}{
\newcommand{\akash}[1]{{\textcolor{purple}{[{\bf Akash:} #1]}}}
\newcommand{\adish}[1]{{\textcolor{red}{[{\bf Adish:} #1]}}}
\newcommand{\yuxin}[1]{{\textcolor{ForestGreen}{[{\bf Yuxin:} #1]}}}
\newcommand{\hank}[1]{{\textcolor{NavyBlue}{[{\bf Hank:} #1]}}}
\newcommand{\annotate}[1]{{\textcolor{blue}{[{\bf Note:} #1]}}}
}
{
\newcommand{\akash}[1]{}
\newcommand{\adish}[1]{}
\newcommand{\yuxin}[1]{}
\newcommand{\hank}[1]{}
\newcommand{\annotate}[1]{}
}

\makeatletter
\newif\if@restonecol
\makeatother

\usepackage[lined, boxed, ruled, commentsnumbered,noend]{algorithm2e}
\usepackage{amsmath,amssymb,xfrac,amsthm}
\setitemize{noitemsep,topsep=1pt,parsep=1pt,partopsep=1pt, leftmargin=12pt}

\newtheorem{lemma}{Lemma}
\newtheorem{theorem}{Theorem}
\newtheorem*{theorem*}{Theorem}

\newtheorem{corollary}{Corollary}
\newtheorem{proposition}{Proposition}
\newtheorem{definition}{Definition}

\newtheorem{assumption}{Assumption}[subsection]

\makeatletter

\newcommand{\Rmnum}[1]{\expandafter\@slowromancap\romannumeral #1@}
\makeatother
\usepackage{caption}

\newcommand{\Rd}{\mathbb{R}^d}
\newcommand{\Rq}{\mathbb{R}^q}

\newcommand{\defref}[1]{Definition~\ref{#1}}

\newcommand{\figref}[1]{Fig.~\ref{#1}}
\newcommand{\eqnref}[1]{\text{Eq.}~(\ref{#1})}
\newcommand{\secref}[1]{\S\ref{#1}}
\newcommand{\appref}[1]{Appendix \ref{#1}}
\newcommand{\stepref}[1]{Step \ref{#1}}
\newcommand{\thmref}[1]{Theorem~\ref{#1}}
\newcommand{\corref}[1]{Corollary~\ref{#1}}
\newcommand{\propref}[1]{Proposition~\ref{#1}}
\newcommand{\lemref}[1]{Lemma~\ref{#1}}

\newcommand{\assref}[1]{Assumption~\ref{#1}}
\newcommand{\assumref}[3]{Assumptions~$[$\ref{#1},\,\ref{#2},\,\ref{#3}$]$}

\newcommand{\paren} [1] {\ensuremath{ \left( {#1} \right) }}
\newcommand{\parenb} [1] {\ensuremath{ \big( {#1} \big) }}
\newcommand{\bigparen} [1] {\ensuremath{ \Big( {#1} \Big) }}

\newcommand{\bracket}[1]{\left[#1\right]}

\newcommand{\curlybracket}[1]{\ensuremath{\left\{#1\right\}}}
\newcommand{\norm}[2]{\ensuremath{\left\langle#1,\:#2\right\rangle_{\cH_{\rK}}}}
\newcommand{\normg}[2]{\ensuremath{\left\langle#1,\:#2\right\rangle}}
\newcommand{\condcurlybracket}[2]{\ensuremath{\left\{#1\::\:#2\right\}}.}
\newcommand{\inmod}[1]{\ensuremath{\left\lvert\left\lvert#1\right\rvert\right\rvert}}

\newcommand{\boldeta}{\ensuremath{\boldsymbol{\eta}}}

\newcommand{\expct}[1]{\mathbb{E}\left[#1\right]}
\newcommand{\expctover}[2]{\mathop{\mathbb{E}}_{#1}\!\left[#2\right]}
\newcommand{\expover}[2]{\mathop{\mathbb{E}}_{#1}\!\Big[#2\Big]}
\newcommand{\abs}[1]{\left\vert#1\right\vert}

\def \argmin {\mathop{\rm arg\,min}}

\newcommand{\bigO}[1]{\ensuremath{\mathcal{O}\paren{#1}}}
\newcommand{\bigTheta}[1]{\ensuremath{\Theta\paren{#1}}}
\newcommand{\bigOmega}[1]{\ensuremath{\Omega\paren{#1}}}

\newcommand{\nats}{\ensuremath{\mathbb{N}}}
\newcommand{\reals}{\ensuremath{\mathbb{R}}}

\newcommand{\cS}{{\mathcal{S}}}

\newcommand{\cA}{{\boldsymbol{\mathcal{A}}}}
\newcommand{\cD}{{\mathcal{D}}}

\newcommand{\cX}{{\mathcal{X}}}
\newcommand{\cH}{{\mathcal{H}}}

\newcommand{\cC}{{\mathcal{C}}}

\newcommand{\cP}{{\mathcal{P}}}

\newcommand{\cY}{{\mathcal{Y}}}

\newcommand{\ba}{{\mathbf{a}}}
\newcommand{\bb}{{\mathbf{b}}}

\newcommand{\be}{{\mathbf{e}}}
\newcommand{\bp}{{\mathbf{p}}}

\newcommand{\bv}{{\mathbf{v}}}

\newcommand{\bx}{{\mathbf{x}}}

\newcommand{\bz}{{\mathbf{z}}}

\newcommand{\rK}{{\mathcal{K}}}

\newcommand{\dd}{\boldsymbol{d}}

\newcommand{\thetab}{\boldsymbol{\theta}}

\renewcommand{\tt}[1]{\textit{#1}}

\newcommand{\boldlam}{\mathbf{\boldsymbol{\Lambda}}}
\def\mathbi#1{\textbf{\em #1}}

\newcommand{\p}{\mathbb{P}}

\def\BState{\State\hskip-\ALG@thistlm}

\newcommand{\id}[1]{\mathds{1}\left\{#1\right\}}
\DeclareMathOperator{\sign}{sign}

\usepackage{natbib}
\setcitestyle{authoryear,round,citesep={;},aysep={,},yysep={;}}

\renewcommand{\cite}[1]{\citep{#1}}

\newtoggle{longversion}
\settoggle{longversion}{true}

\title{The Teaching Dimension of Kernel Perceptrons}

\begin{document}

\twocolumn[

\aistatstitle{The Teaching Dimension of Kernel Perceptron}

\aistatsauthor{ Akash Kumar \And Hanqi Zhang \And  Adish Singla \And Yuxin Chen}

\aistatsaddress{ MPI-SWS
 \And University of Chicago
 \And MPI-SWS
 \And University of Chicago
 } ]

\begin{abstract}
Algorithmic machine teaching has been studied under the linear setting where exact teaching is possible. However, little is known for teaching nonlinear learners. Here, we establish the sample complexity of teaching, aka teaching dimension, for kernelized perceptrons for different families of feature maps. As a warm-up, we show that the teaching complexity is $\Theta(d)$ for the exact teaching of linear perceptrons in $\mathbb{R}^d$, and $\Theta(d^k)$ for kernel perceptron with a polynomial kernel of order $k$. Furthermore, under certain smooth assumptions on the data distribution, we establish a rigorous bound on the complexity for approximately teaching a Gaussian kernel perceptron. We provide numerical examples of the optimal (approximate) teaching set under several canonical settings for linear, polynomial and Gaussian kernel perceptrons.
\end{abstract}
\section{Introduction}
%
Machine teaching studies the problem of finding an optimal training sequence to steer a learner towards a target concept \cite{DBLP:journals/corr/ZhuSingla18}. 
An important learning-theoretic complexity measure of machine teaching is the \emph{teaching dimension} \cite{goldman1995complexity}, which specifies the minimal number of training examples required in the worst case to teach a target concept. Over the past few decades, the notion of teaching dimension has been investigated under a variety of learner's models and teaching protocols (e.g,. \citet{cakmak2012algorithmic,singla2013actively,singla2014near,liu2017iterative,haug2018teaching,tschiatschek2019learner,DBLP:conf/icml/LiuDLLRS18,DBLP:conf/ijcai/KamalarubanDCS19,DBLP:conf/nips/Hunziker0AR0PYS19,DBLP:conf/ijcai/DevidzeMH0S20,DBLP:conf/icml/RakhshaRD0S20}).
One of the most studied scenarios is the case of teaching a version-space learner \cite{goldman1995complexity,article:anthony95,zilles2008teaching,doliwa2014recursive,chen2018understanding,mansouri2019preference,pmlr-v98-kirkpatrick19a}. Upon receiving a sequence of training examples from the teacher, a version-space learner maintains a set of hypotheses that are consistent with the training examples, and outputs a \emph{random} hypothesis from this set. 

As a canonical example, consider teaching a 1-dimensional binary threshold function $f_{\theta^*}(x) = \id{x -\theta^*}$ for $x\in [0,1]$. For a learner with a finite (or countable infinite) version space, e.g., $\theta \in \{\frac{i}{n}\}_{i=0,\dots,n}$ where $n\in \mathbb{Z}^+$ (see \figref{fig:example.1d-vs}), a smallest training set is $\{\left(\frac{i}{n},0\right), \left(\frac{i+1}{n},1\right)\}$ where $\frac{i}{n} \leq \theta^* <  \frac{i+1}{n}$; thus the teaching dimension is $2$. However, when the version space is continuous, the teaching dimension becomes $\infty$, because it is no longer possible for the learner to pick out a unique threshold $\theta^*$ with a finite training set. This is due to two key (limiting) modeling assumptions of the version-space learner: (1) all (consistent) hypotheses in the version space are treated equally, and (2) there exists a hypothesis in the version space that is consistent with all training examples. As one can see, these assumptions fail to capture the behavior of many modern learning algorithms, where the best hypotheses are often selected via \emph{optimizing} certain loss functions, and the data is not perfectly separable (i.e. not realizable w.r.t. the hypothesis/model class).

To lift these modeling assumptions, a more realistic teaching scenario is to consider the learner as an \emph{empirical risk minimizer} (ERM). In fact, under the realizable setting, the version-space learner could be viewed as an ERM that optimizes the 0-1 loss---one that finds all hypotheses with zero training error. Recently, \citet{JMLR:v17:15-630} studied the teaching dimension of 
linear ERM, and established values of teaching dimension for several classes of linear (regularized) ERM learners, including support vector machine (SVM), logistic regression and ridge regression. As illustrated in \figref{fig:example.1d-hinge}, for the previous example it suffices to use $\{\left(\theta^*-\epsilon,0\right), \left(\theta^*+\epsilon,1\right)\}$  with any $\epsilon \leq \min(1-\theta^*, \theta^*)$  as training set to teach $\theta^*$ as an optimizer of the SVM objective (i.e., $l$2 regularized hinge loss); hence the teaching dimension is 2.  In \figref{fig:example.1d-perceptron}, we consider teaching an ERM learner with perceptron loss, i.e., 
$\ell(f_\theta(x), y) = \max\left( -y\cdot (x-\theta), 0\right)$ (where $y \in \curlybracket{-1,1}$). If the teacher is allowed to construct \emph{any} training example with \emph{any} labeling\footnote{If the teacher is restricted to only provide consistent labels (i.e., the realizable setting), then the ERM with perceptron loss reduces to the version space learner, where the teaching dimension is $\infty$.} , then it is easy to verify that the minimal training set is $\{(\theta^*, -1), (\theta^*,1)\}$.

\begin{figure}[t]
\centering
	\begin{subfigure}[b]{.2\textwidth}
	   \centering
		\includegraphics[trim={0, 0, 0, 10mm}, width=\linewidth]{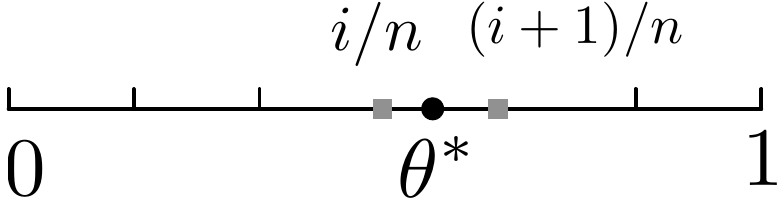}
		\vspace{-5mm}
		\caption{$\textsc{0/1}$ loss}
		\label{fig:example.1d-vs}
	\end{subfigure}\qquad
	\begin{subfigure}[b]{.2\textwidth}
	    \centering
		\includegraphics[width=\linewidth]{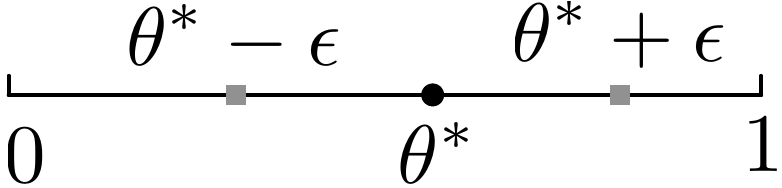}
		\vspace{-5mm}
		\caption{SVM (hinge loss)}
		\label{fig:example.1d-hinge}
	\end{subfigure}\\
	\vspace{2mm}
	\begin{subfigure}[b]{.2\textwidth}
	    \centering
		\includegraphics[width=\linewidth]{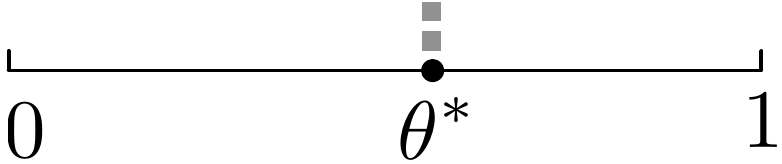}
		\vspace{-5mm}
		\caption{Perceptron}
		\label{fig:example.1d-perceptron}
	\end{subfigure}	
	\caption{Teaching a 1D threshold function to an ERM learner. Training instances are marked in grey. (a) Version-space learner with a finite hypothesis set. (b) SVM and training set $\{\left(\theta^*-\epsilon,0\right), \left(\theta^*+\epsilon,1\right)\}$.  (c) ERM learner with (perceptron) loss and training set $\{(\theta^*, 0), (\theta^*,1)\}$. }\label{fig:illu-relaxed-general}
	\vspace{-3mm}
\end{figure}
While these results show promise at understanding optimal teaching for ERM learners, existing work \cite{JMLR:v17:15-630} has focused exclusively on the linear setting with the goal to teach the exact hypothesis (e.g., teaching the exact model parameters or the exact decision boundary for classification tasks). Aligned with these results, we establish an upper bound as shown in \secref{subsec.linear}. It remains a fundamental challenge to rigorously characterize the teaching complexity for nonlinear learners. Furthermore, in the cases where exact teaching is not possible with a finite training set, the classical teaching dimension no longer captures the fine-grained complexity of the teaching tasks, and hence one needs to relax the teaching  goals and investigate new notions of teaching complexity.


In this paper, we aim to address the above challenges. We focus on kernel perceptron, a specific type of ERM learner that is less understood even under the linear setting.
Following the convention in teaching ERM learners, we consider the \emph{constructive} setting, where the teacher can construct arbitrary teaching examples in the support of the data distribution. Our contributions are highlighted below, with main theoretical results summarized in Table~\ref{tab:results-overview}.

\begin{itemize}
\item We formally define approximate teaching of kernel perceptron, and propose a novel measure of teaching complexity, namely the \emph{$\epsilon$-approximate teaching dimension} ($\epsilon$-TD), which captures the complexity of teaching a ``relaxed'' target that is close to the target hypothesis in terms of the expected risk.  Our relaxed notion of teaching dimension strictly generalizes the teaching dimension of \citet{JMLR:v17:15-630}, where it trades off the teaching complexity against the risk of the taught hypothesis, and hence is more practical in characterizing the complexity of a teaching task (\secref{sec.statement}).\newline

\item We show that exact teaching is feasible for kernel perceptrons with finite dimensional feature maps, such as linear kernel and polynomial kernel. Specifically, for data points in $\Rd$, we establish a $\bigTheta{d}$ bound on the teaching dimension of linear perceptron. 
Under a mild condition on data distribution, we provide a tight bound of $\bigTheta{\binom{d+k-1}{k}}$ for polynomial perceptron of order $k$. 
We also exhibit optimal training sets that match these teaching dimensions (\secref{subsec.linear} and \secref{subsec.poly}).\newline

\item We further show that for Gaussian kernelized perceptron, exact teaching is not possible with a finite set of hypotheses, and then establish a $d^{\bigO{\log^2 \frac{1}{\epsilon}}}$ bound on the $\epsilon$-approximate teaching dimension (\secref{subsec.gaussiankernel}). To the best of our knowledge, these results constitute the first known bounds on (approximately) teaching a non-linear ERM learner (\secref{sec.theoreticalresults}). 
\end{itemize}

\begin{table}[t!]
  \centering
  \scalebox{.88}{
\begin{tabular}{cccc}
\toprule
 & \textbf{linear} & \textbf{polynomial}           & \textbf{Gaussian}                            \\ 
 \midrule
TD (exact)     & $\bigTheta{d}$   & $\bigTheta{\binom{d+k-1}{k}}$ & $\infty$   \\
$\epsilon$-approximate TD & - & - & $d^{\bigO{\log^2 \frac{1}{\epsilon}}}$   \\
\textbf{Assumption }      & -                &   \ref{assumption: polyorthogonal}  &  \ref{assumption: orthogonal}, \ref{assumption: bounded cone}\\
\bottomrule
\end{tabular}
}
\caption{Teaching dimension for kernel perceptron}\label{tab:results-overview}
\vspace{-3mm}
\end{table}

\section{Problem Statement}\label{sec.statement}
\paragraph{Basic definitions}

We denote by $\cX$ the input space and $\cY:= \{-1,1\}$ the output space. A hypothesis is a function $h: \cX \to \cY$. In this paper, we identify a hypothesis $h_{\thetab}$ with its model parameter $\thetab$. The hypothesis space $\cH$ is a set of hypotheses. By training point we mean a pair $\paren{\bx,y} \in \cX \times \cY$. We assume that the training points are drawn from an unknown distribution $\cP$ over $\cX \times \cY$. A training set is a multiset $\cD$ = $\curlybracket{\paren{\bx_1,y_1},\cdots,\paren{\bx_n,y_n}}$ where repeated pairs are allowed. Let $\mathbb{D}$ denote the set of all training sets of all sizes. A learning algorithm $\cA: \mathbb{D} \to 2^{\cH} $ takes in a training set $D \in \mathbb{D}$ and outputs a subset of the hypothesis space $\cH$. That is, $\cA$ doesn't necessarily return a unique hypothesis. 
\vspace{-1mm}
\paragraph{Kernel perceptron}
Consider a set of training points $\cD := \curlybracket{\paren{\bx_i, y_i}}_{i=1}^n$ where $\bx_i \in \reals^d$  and hypothesis $\thetab \in \Rd$. A linear perceptron is defined as $f_{\thetab}(\bx):= \sign(\thetab\cdot \bx)$ 
in homogeneous setting. We consider the algorithm $\cA_{opt}$ to learn an optimal perceptron to classify $\cD$ as defined below:
\begin{equation}
 \cA_{opt}\paren{\cD} := \argmin_{\thetab \in \reals^d} \sum_{i = 1}^n \ell(f_{\thetab}(\bx_i),y_i).   \label{eqn: objectmain}
\end{equation}
where the loss function $\ell(f_{\thetab}(\bx),y) := \max(-y\cdot f_{\thetab}(\bx), 0)$. Similarly, we consider the non-linear setting via kernel-based hypotheses for perceptrons that are defined with respect to a kernel operator $\rK: \cX \times \cX \to \reals$ which adheres to Mercer’s positive definite conditions \cite{vapnik1998statistical}. A kernel-based hypothesis has the form, 
\begin{equation}
    f(\bx) = \sum_{i=1}^k{\alpha_i}\cdot\rK(\bx_i,\bx) \label{eqn: kernelfunction}
\end{equation}
where $\forall i\,\, \bx_i \in \cX$ and $\alpha_i$ are reals. In order to simplify the derivation of the algorithms and their analysis, we associate a \tt{reproducing kernel Hilbert space}  (RKHS)  with $\rK$ in  the  standard  way  common  to  all  kernel  methods.  Formally, let $\cH_{\rK}$ be  the  closure  of  the  set  of  all  hypotheses  of  the  form  given  in \eqnref{eqn: kernelfunction}. A non-linear kernel perceptron corresponding to $\rK$ optimizes \eqnref{eqn: objectmain} as follows:
\begin{equation}
    \cA_{opt}(\cD):= \argmin_{\thetab \in \cH_{\rK}}\sum_{i=1}^n \ell(f_{\thetab}(\bx_i),y_i)\label{eqn: objectkernel}
\end{equation}
where $f_{\thetab}(\cdot) = \sum_{i=1}^l \alpha_i\cdot \rK(\ba_i,\cdot)$ for some $\{\ba_i\}_{i=1}^l \subset \cX$ and $\alpha_i$ real. Alternatively, we also write $f_{\thetab}(\cdot) = \thetab\cdot\Phi(\cdot)$ where
$\Phi : \cX \rightarrow \cH_{\rK}$ is defined as feature map to the kernel function $\rK$. A reproducing kernel Hilbert space with $\rK$ could be decomposed as $\rK(\bx,\bx') = \normg{\Phi(\bx)}{\Phi(\bx')}$ \cite{learnkernel} for any $\bx, \bx' \in \cX$.
Thus, we also identify $f_{\thetab}$ as $\sum_{i=1}^l \alpha_i\cdot \Phi(\ba_i)$.
\paragraph{The teaching problem}
We are interested in the problem of teaching a target hypothesis $\thetab^*$ where a helpful \tt{teacher} provides labelled data points $\mathcal{TS} \subseteq \cX\times \cY$, also defined as a \tt{teaching set}. Assuming the constructive setting \cite{JMLR:v17:15-630}, to teach a kernel perceptron learner the teacher can construct a training set with any items 
in $\Rd$ i.e. for any $(\bx', y') \in \mathcal{TS}$ we have  $\bx' \in \Rd$ and $y' \in \curlybracket{-1,1}$.
Importantly, for the purpose of teaching we do not \tt{assume} that $\mathcal{TS}$ are drawn \tt{i.i.d} from a distribution. We define the teaching dimension for \tt{exact} parameter of $\thetab^*$ corresponding to a kernel perceptron as $TD(\thetab^*, \cA_{opt})$, which is the size of the smallest teaching set $\mathcal{TS}$ such that $\cA_{opt}\paren{\mathcal{TS}} = \{\thetab^*\}$. We define teaching of exact parameters of a target hypothesis $\thetab^*$ as \tt{exact teaching}. Since, a perceptron is agnostic to norms, we study the problem of teaching a target classifier \tt{decision boundary} where $\cA_{opt}\paren{\mathcal{TS}} = \{t\thetab^*\}$ for some real $t > 0$. Thus, $$TD(\{t\thetab^*\}, \cA_{opt}) = \min_{\text{real}\,p > 0} TD(p\thetab^*, \cA_{opt}).$$ 
Since it can be stringent to construct a teaching set for decision boundary (see \secref{subsec.gaussiankernel}), exact teaching is not always feasible. We introduce and study \tt{approximate teaching} which is formally defined as:
\begin{definition}[$\epsilon$-approximate teaching set]
Consider a kernel perceptron learner, with a kernel $\rK: \cX \times \cX \to \reals$ and the corresponding RKHS feature map $\Phi(\cdot)$. For a target model $\thetab^* \in \cH_{\rK}$ and $\epsilon > 0$, we say $\mathcal{TS} \subseteq \cX\times \cY$ is an $\epsilon$-approximate teaching set wrt to $\cP$ if the kernel perceptron $\hat{\thetab} \in \cA_{opt}(\mathcal{TS})$ 
satisfies 
\looseness -1
\begin{equation}
    \left|\expct{\max(-y\cdot f^*(\bx), 0)} - \expct{\max(-y\cdot \hat{f}(\bx), 0)}\right| \le \epsilon
\end{equation}
where the expectations are over $(\bx,y)\sim \cP$ 
and $f^*(\bx) = \thetab^*\cdot \Phi(\bx)$ and $\hat{f}(\bx) = \hat{\thetab}\cdot \Phi(\bx)$.
\end{definition}
Naturally, we define approximate teaching dimension as:
\begin{definition}[$\epsilon$-approximate teaching dimension]\label{def:approxteaching}
Consider a kernel perceptron learner, with a kernel $\rK: \cX \times \cX \to \reals$ and the corresponding RKHS feature map $\Phi(\cdot)$. For a target model $\thetab^* \in \cH_{\rK}$ and $\epsilon > 0$, we define $\epsilon$-$TD(\thetab^*,\cA_{opt})$ as the teaching dimension which is the size of the smallest teaching set for $\epsilon$-approximate teaching of $\thetab^*$ wrt $\cP$.
\end{definition}
According to \defref{def:approxteaching}, exact teaching corresponds to constructing a $0$-approximate teaching set for a target classifier (e.g., the decision boundary of a kernel perceptron).
We study linear and polynomial kernelized perceptrons in the exact teaching setting. Under some mild assumptions on the smoothness of the data distribution, we establish approximate teaching bound on approximate teaching dimension for Gaussian kernelized perceptron.

\vspace{-2mm}
\section{Teaching Dimension for Kernel Perceptron}\label{sec.theoreticalresults}
\vspace{-1mm}
In this section, we study the generic problem of teaching kernel perceptrons in three different settings:\,1) linear (in \secref{subsec.linear}); 2)\, polynomial (in \secref{subsec.poly}); and Gaussian (in \secref{subsec.gaussiankernel}). 
Before establishing our main result for Gaussian kernelized perceptrons, we first introduce two important results for linear and polynomial perceptrons inherently connected to the Gaussian perceptron. 
Our proofs are inspired by ideas from linear algebra and projective geometry as detailed \iftoggle{longversion}{in \appref{appendix:table-of-contents}}{in the supplemental materials}.
\vspace{-1mm}
\subsection{Homogeneous Linear Perceptron}\label{subsec.linear}
In this subsection, we study the problem of teaching a linear perceptron. 
First, we consider an optimization problem similar to \eqnref{eqn: objectmain} as shown in \citet{JMLR:v17:15-630}:\vspace{-1mm}
\begin{equation}
\cA_{opt} := \argmin_{\thetab \in \Rd} \sum_{i = 1}^n \ell(\thetab\cdot{\bx}_i, y_i) + \frac{\lambda}{2}||\thetab||^2_{A} \label{eqn: eqn1}
\looseness -3
\end{equation}
where $\ell(\cdot,\cdot)$ is a convex loss function, $A$ is a positive semi-definite matrix, $||\thetab||_{A}$ is defined as $\sqrt{\thetab^\top A \thetab}$, 
and $\lambda > 0$. For convex loss function $\ell(\cdot,\cdot)$, Theorem 1~\cite{JMLR:v17:15-630} established a degree-of-freedom lower bound on the number of training items to obtain a unique solution $\thetab^*$.
Since, the loss function for linear perceptron is convex thus we immediately obtain a lower bound on the teaching dimension as follows:
\begin{corollary}\label{cor: linear lower bound}
If $A = 0$ and $\lambda = 1$, then \eqnref{eqn: objectmain} can be solved as \eqnref{eqn: eqn1}. Moreover, teaching dimension for decision boundary corresponding to a target model $\thetab^*$ is lower-bounded by $\bigOmega{d}$.
\end{corollary}
Now, we would establish an upper bound on $TD(\cA_{opt},\thetab^*)$ for exact teaching of the decision boundary of a 
target model $\thetab^*$. The key idea is to find a set of points which span the orthogonal subspace of $\thetab^*$, which we use to force a solution $\hat{\thetab} \in \cA_{opt}$ such that it has a component only along $\thetab^*$. Formally, we state the claim of the result with proof as follows:
\begin{theorem}\label{thm: linear perceptron main result}
Given any target model $\thetab^*$, for solving \eqnref{eqn: objectmain} the teaching dimension for the decision boundary corresponding to $\thetab^*$ is $\bigTheta{d}$. The following is a teaching set:
\begin{align*}
    {\bx}_i = \bv_i,\quad y_i = 1\quad \forall\; i\; \in\; [d-1];\qquad\qquad\qquad\ \ \,\\
    \quad {\bx}_d = -\sum_{i=1}^{d-1} \bv_i,\quad y_d = 1;\quad {\bx}_{d+1} = \thetab^*,\quad y_{d+1} = 1
\end{align*}
where $\{\bv_i\}_{i=1}^{d}$ is an orthogonal basis for $\Rd$ which extends with $\bv_d = \thetab^*$.
\vspace{-1mm}
\end{theorem}
\begin{proof}
Using \corref{cor: linear lower bound}, the lower bound for solving \eqnref{eqn: objectmain} is immediate. Thus, if we show that the mentioned labeled set of training points form a teaching set, then we can show an upper bound which would imply a tight bound of $\bigTheta{d}$ on the teaching dimension for finding the decision boundary. Denote the set of labeled data points as $\cD$. Denote by $\bp(\thetab) := \sum_{i = 1}^{d+1} \max(-y_i\cdot\thetab\cdot{\bx}_i,\: 0)$. Since $\{\bv_i\}_{i=1}^{d}$ is an orthogonal basis, thus $\forall \, i \in [d-1]\quad \bv_i\cdot \thetab^* = 0$, thus it is not very difficult to show that $\bp(t\thetab^*) = 0$ for some positive scalar $t$.
Note, if $\hat{\thetab}$ is a solution to \eqnref{eqn: objectmain} then:
\vspace{-1mm}
\begin{equation*}
    \hat{\thetab} \in \argmin_{\thetab \in \Rd} \sum_{i = 1}^{d+1} \max(-y_i\cdot\thetab\cdot{\bx}_i,\: 0)
\end{equation*}
Also, $\bp(\hat{\thetab}) = 0 \implies {\bx}_i\cdot \hat{\thetab} \ge 0\; \forall \, i \in [d]$ but then $\bx_{d} = -  \sum_{i=1}^{d-1} {\bx}_i$ $\implies \forall \, i \in [d]\quad {\bx}_i\cdot \hat{\thetab} = 0$. Note that, $\hat{\thetab}\cdot \thetab^* \ge 0$ forces $\hat{\thetab} = t\thetab^*$ for some positive constant $t$. Thus, $\cD$ is a teaching set for the decision boundary of $\thetab^*$. This establishes the upper bound, and hence the theorem follows.
\end{proof}
\vspace{-3mm}
\paragraph{Numerical example} To illustrate \thmref{thm: linear perceptron main result}, we provide a numerical example for teaching a linear perceptron in $\mathbb{R}^3$, with $\thetab^* = (-3,3,5)^\top$ (illustrated in \figref{fig:exp:exact-teaching:linear}). To construct the teaching set, we first obtain an orthogonal basis $\{(0.46, 0.86, -0.24)^\top, (0.76, -0.24, 0.6)^\top\}$ for the subspace orthogonal to $\thetab^*$, and add a vector $(-1.22, -0.62, -0.36)^\top$ which is in the exact opposite direction of the first two combined. Finally we add to $\mathcal{TS}$ an arbitrary vector which has a positive dot product with the normal vector, e.g. $(-0.46, 0.46, 0.76)^\top$. Labeling all examples positive, we obtain $\mathcal{TS}$ of size $4$. 

\subsection{Homogeneous Polynomial Kernelized Perceptron}\label{subsec.poly}
In this subsection, we study the problem of teaching a polynomial kernelized perceptron in realizable setting. Similar to \secref{subsec.linear}, we establish an exact teaching bound on the teaching dimension under a mild condition on the data distribution.
We consider homogeneous polynomial kernel $\rK$ of degree $k$ in which for any $\bx, \bx' \in \Rd$ \[\rK(\bx, \bx') = \paren{\langle\bx, \bx'\rangle}^k\]
If $\Phi(\cdot)$ denotes the \textit{feature map} for the corresponding RKHS, then we know that the dimension of the map is $\binom{d+k-1}{k}$ where each component of the map can be represented by $\Phi_{\boldsymbol{\lambda}}(\bx) = \sqrt{ \frac{k!}{\prod_{i=1}^d\boldsymbol{\lambda}_i!}}\bx^{\boldsymbol{\lambda}}$ 
where $\boldsymbol{\lambda} \in \paren{\nats\cup \curlybracket{0}}^{d}$ and $\sum_{i} \boldsymbol{\lambda}_i = k$. 
Denote by $\cH_{\rK}$ the RKHS corresponding to the polynomial kernel $\rK$. We use $\cH_k := \cH_k(\Rd)$ to represent the linear space of homogeneous polynomials of degree $k$ over $\Rd$. We mention an important result which shows the RKHS for polynomial kernels is isomorphic to the space of homogeneous polynomials of degree $k$ in $d$ variables.
\begin{proposition}[Chapter III.2, Proposition 6 \cite{article}]\label{prop: polynomial space}
$\cH_k = \cH_{\rK}$ as function spaces and inner product spaces.
\end{proposition}
The dimension  $\dim \paren{\cH_k(\Rd)}$ of the linear space of homogeneous polynomials of degree $k$ over $\Rd$ is $\binom{d+k-1}{k}$. Denote by $r := \binom{d+k-1}{k}$. Since $\cH_{\rK}$ is a vector space for polynomial kernel $\rK$, thus for exact teaching there is an obvious lower bound of $\bigOmega{\binom{d+k-1}{k}}$ on the teaching dimension. 

Before we establish the main result of this subsection we state a mild assumption on the target model we consider for exact teaching which is as follows:
\begin{assumption}[Existence of orthogonal polynomials]\label{assumption: polyorthogonal} For the target model $\thetab^* \in \cH_{\rK}$, we assume that there exist $(r-1)$ linearly independent polynomials on the orthogonal subspace of $\thetab^*$ in $\cH_{\rK}$ of the form $\left\{\Phi(\bz_i)\right\}_{i=1}^{r-1}$ where $\forall i\; \bz_i \in \cX $. 
\end{assumption} 
Similar to \thmref{thm: linear perceptron main result}, the key insight in having \assref{assumption: polyorthogonal} is to find independent polynomial on the orthogonal subspace defined by $\thetab^*$. We state the claim here with proof established  \iftoggle{longversion}{in \appref{appendix: polynomial perceptron}}{in the supplemental materials}.
\begin{theorem}\label{thm: poly_main_theorem}
For all target models $\thetab^* \in \cH_{\rK}$ for which the \assref{assumption: polyorthogonal} holds, for solving \eqnref{eqn: objectkernel}, the exact teaching dimension for the decision boundary corresponding to $\thetab^*$ is $\bigO{\binom{d+k-1}{k}}$. 
\end{theorem}
\begin{figure*}[t]
\centering
		\begin{subfigure}[b]{0.3\textwidth}
	   \centering
		\includegraphics[width=\linewidth]{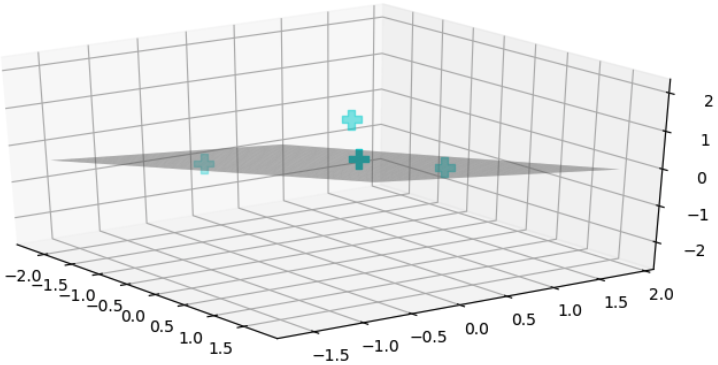}
		\caption{Linear ($\mathcal{TS}$)}\label{fig:exp:exact-teaching:linear}
		\label{fig:example.polytope}
	\end{subfigure}
		\begin{subfigure}[b]{0.3\textwidth}
	    \centering
		\includegraphics[width=\linewidth]{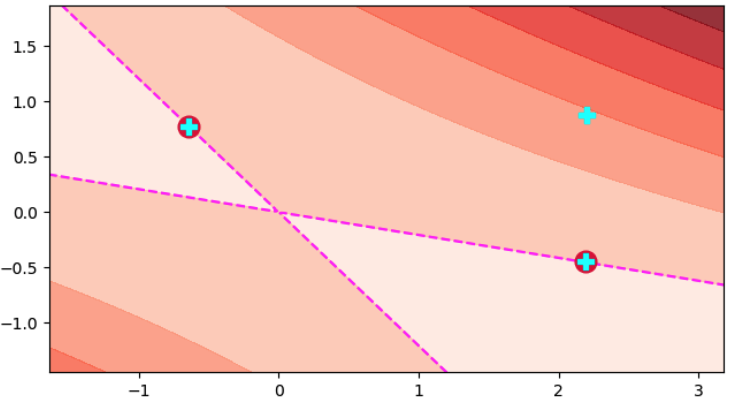}
		\caption{Polynomial ($\mathcal{TS}$)}\label{fig:exp:exact-teaching:polynomial}
	\end{subfigure}	
	\begin{subfigure}[b]{0.3\textwidth}
	    \centering
		\includegraphics[width=\linewidth]{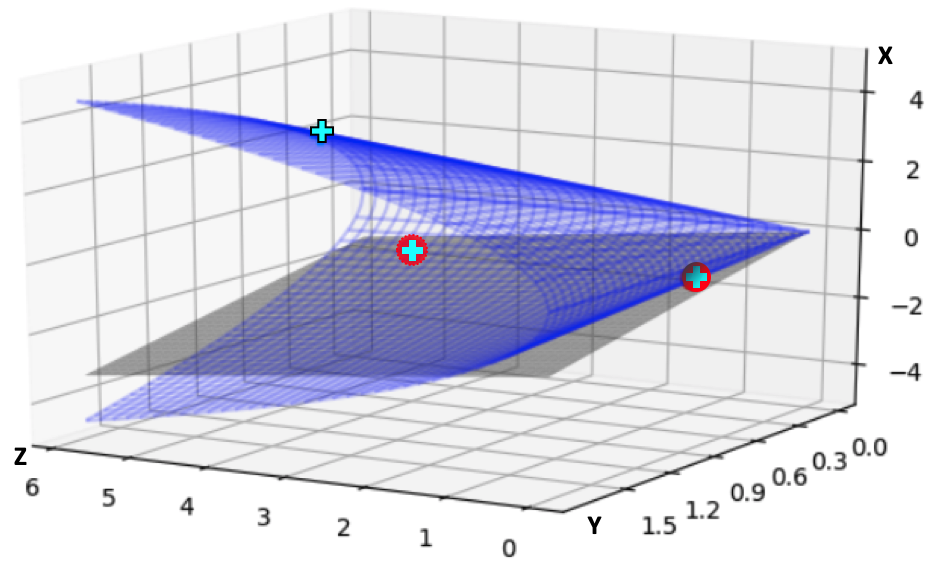}
		\caption{Polynomial (feature space)} \label{fig:exp:feature-space:polynomial}
	\end{subfigure}	
  \caption{Numerical examples of exact teaching for linear and polynomial perceptrons. Cyan plus marks and red dots correspond to positive and negative teaching examples respectively.
  }
	\label{fig:exp:exact-teahcing}
\end{figure*}
\vspace{-2mm}
\paragraph{Numerical example} 
For constructing $\mathcal{TS}$ in the polynomial case, we follow a similar strategy in the higher dimensional space that the original data is projected into. The only difference is that we need to ensure the teaching examples have pre-images in the original space. For that, we adopt a randomized algorithm that solves for $r-1$ boundary points in the original space (i.e. solve for $\theta^*\cdot \Phi(\mathbf{x}) = 0$) , while checking the images of these points are linearly independent. Also, instead of adding a vector in the opposite direction of these points combined, we simply repeat the $r-1$ points in the teaching set, while assigning one copy of them positive labels and the other copy negative labels. Finally, we need one last vector (label it positive) whose image has a positive component in $\theta^*$, and we obtain $\mathcal{TS}$ of size $2r-1$.

\figref{fig:exp:exact-teaching:polynomial} and  \figref{fig:exp:feature-space:polynomial} demonstrate the above constructive procedure on a numerical example with $d=2$, homogeneous polynomial kernel of degree 2, and $\thetab^* = {(1, 4, 4)}^\top$. In \figref{fig:exp:exact-teaching:polynomial} we show the decision boundary (red lines) and the level sets (polynomial contours) of this quadratic perceptron, as well as the teaching set identified via the above algorithmic procedure. In \figref{fig:exp:exact-teaching:polynomial}, we visualize the decision boundary (grey plane) in the feature space (after applying the feature map). The blue surface corresponds to all the data points that have pre-images in the original space $\mathbb{R}^2$.

\subsection{Limitations in Exact Teaching of Polynomial Kernel Perceptron}\label{subsec.motivation}
In the previous section \secref{subsec.poly}, we imposed the \assref{assumption: polyorthogonal} on the target models $\thetab^*$. It turns out that we couldn't do better than this. More concretely, we need to impose this assumption for exact teaching of polynomial kernel perceptron learner. Further, there are pathological cases where violation of the assumption leads to models which couldn't be approximately taught. 

Intuitively, solving \eqnref{eqn: objectkernel} in the paradigm of exact teaching reduces to nullifying the orthogonal subspace of $\thetab^*$ i.e. any component of $\thetab^*$ along the subspace is nullified. Since the information of the span of the subspace has to be encoded into the datapoints chosen for teaching, \assref{assumption: polyorthogonal} is a natural step to make. Interestingly, we show that the step is not so stringent.
In the realizable setting in which all the teaching points are correctly classified, if we lift the assumption then exact teaching is not possible.
We state the claim in the following lemma:
\begin{lemma}\label{lemma: exact teaching}
Consider a target model $\thetab^*$ that doesn't satisfy \assref{assumption: polyorthogonal}. Then, there doesn't exist a teaching set $\mathcal{TS}_{\thetab^*}$ which exactly teaches $\thetab^*$ i.e. for any $\mathcal{TS}_{\thetab^*}$ and any real $t > 0$ $$\cA_{opt}\paren{\mathcal{TS}_{\thetab^*}} \neq \{t\thetab^*\}.$$
\end{lemma}
\lemref{lemma: exact teaching} shows that for \tt{exact} teaching $\thetab^*$ should satisfy \assref{assumption: polyorthogonal}. Then, the natural question that arises is whether we can achieve arbitrarily $\epsilon$-close \tt{approximate} teaching for $\thetab^*$. In other words, we would like to find $\Tilde{\thetab}^*$ that satisfies \assref{assumption: polyorthogonal} and is in $\epsilon$-neighbourhood of $\thetab^*$. We show a negative result for this when $k$ is even. For this we assume that, the datapoints in the teaching set $\mathcal{TS}_{\Tilde{\thetab}^*}$ have lower-bounded norm, call it, $\delta > 0$ i.e. if $({\bx}_i,y_i) \in \mathcal{TS}_{\Tilde{\thetab}^*}$ then $||\Phi({\bx}_i)|| \ge \delta$. We require this additional assumption only for the purpose of analysis. We would show that it wouldn't lead to any pathological cases where the constructed target model $\thetab^*$ incorporates approximate teaching.
\begin{lemma}\label{lemma: approximate teaching }
Let $\cX \subseteq \Rd$ and $\cH_{\rK}$ be the reproducing kernel Hilbert space such that kernel function $\rK$ is of degree $k$. If $k$ has parity even then there exists a target model $\thetab^*$ which violates \assref{assumption: polyorthogonal} and can't be taught approximately.
\end{lemma}

The results are discussed in details with proofs \iftoggle{longversion}{in \appref{appendix: motivation}}{in the supplemental materials}. \assref{assumption: polyorthogonal} and the stated lemmas provide insights into understanding the problem of teaching for non-linear perceptron kernels. In the next section, we study Gaussian kernel and the ideas generated here would be useful in devising a teaching set in the paradigm of approximate teaching.
\subsection{Gaussian Kernelized Perceptron}\label{subsec.gaussiankernel}
In this subsection, we consider the Gaussian kernel. Under mild assumptions inspired by the analysis of teaching dimension for exact teaching of linear and polynomial kernel perceptrons, we would establish as our main result an upper bound on the $\epsilon$-approximate teaching dimension of Gaussian kernel perceptrons using a construction of an $\epsilon$-approximate teaching set.
\paragraph{Preliminaries of Gaussian kernel} A Gaussian kernel $\rK$ is a function of the form
\begin{equation}
\rK(\bx, \bx') = \mathbi{e}^{-\frac{||\bx-\bx'||^2}{2\sigma^2}} \label{eqn:eqn11}
\end{equation}
for any $\bx, \bx' \in \Rd$ and parameter $\sigma$. First, we would try to understand the feature map before we find an approximation to it. Notice:
\[\mathbi{e}^{-\frac{||\bx-\bx'||^2}{2\sigma^2}} = \mathbi{e}^{-\frac{||\bx||^2}{2\sigma^2}}\mathbi{e}^{-\frac{||\bx'||^2}{2\sigma^2}}\mathbi{e}^{\frac{\normg{\bx}{\bx'}}{\sigma^2}}  \]
Consider the scalar term $z = \normg{\bx}{\bx'}/\sigma^2$. We can expand the term of the product using the Taylor expansion of $\mathbi{e}^z$ near $z = 0$ as shown in \citet{Cotter2011ExplicitAO}, which amounts to
$\mathbi{e}^{\frac{\normg{\bx}{\bx'}}{\sigma^2}} = \sum_{k=0}^{\infty}\frac{1}{k!}\paren{\frac{\normg{\bx}{\bx'}}{\sigma^2}}^k$. We can further expand the previous sum as
\begin{align}
\mathbi{e}^{\frac{\normg{\bx}{\bx'}}{\sigma^2}} &= \sum_{k=0}^{\infty}\frac{1}{k!}\paren{\frac{\normg{\bx}{\bx'}}{\sigma^2}}^k \nonumber\\
&= \sum_{k=0}^{\infty}\frac{1}{k!\sigma^{2k}}\bigparen{\sum_{l=1}^d \bx_l\cdot\bx'_l}^k \nonumber\\
&=  \sum_{k=0}^{\infty}\frac{1}{k!\sigma^{2k}}\sum_{|\boldsymbol{\lambda}|=k} \cC^k_{\boldsymbol{\lambda}}\cdot\bx^{\boldsymbol{\lambda}}\cdot(\bx')^{\boldsymbol{\lambda}} \label{eqn:eqn12}
\end{align}
where $\cC^k_{\boldsymbol{\lambda}} = \frac{k!}{\prod_{i=1}^d\boldsymbol{\lambda}_i!}$. Thus, we use \eqnref{eqn:eqn12} to obtain explicit feature representation to the Gaussian kernel in \eqnref{eqn:eqn11} as 
$\Phi_{k,\boldsymbol{\lambda}}(\bx) = \mathbi{e}^{-\frac{||\bx||^2}{2\sigma^2}}\cdot \frac{\sqrt{\cC^k_{\boldsymbol{\lambda}}}}{\sqrt{k!}\sigma^k}\cdot\bx^{\boldsymbol{\lambda}}$. 
We get the explicit feature map $\Phi(\cdot)$ for the Gaussian kernel with coordinates as specified. Theorem 1 of \citet{haquangminh} characterizes the RKHS of Gaussian kernel. It establishes that $\dim(\cH_{\rK}) = \infty$.
Thus, we note that the exact teaching for an arbitrary target classifier $f^*$ in this setting has an infinite lower bound. This calls for analysing the teaching problem of a Gaussian kernel in the \tt{approximate} teaching setting.
\paragraph{Definitions and notations for approximate teaching} 
For any classifier $f \in \cH_{\rK}$,
we define $\textbf{err}$($f$) = $\expctover{(\bx,y) \sim \cP(\bx,y)}{\max(-y\cdot f(\bx),0)}$. Our goal is to find a classifier $f$ with the property that its expected true loss $\textbf{err}$($f$) is as small as possible. 
In the realizable setting, we assume that there exists an optimal separator $f^*$ such that for any data instances sampled from the data distribution the labels are consistent i.e. $\cP(y\cdot f^*(\bx) \le 0) = 0$. In addition, we also experiment for the non-realizable setting.
In the rest of the subsection, we would study the relationship between the teaching complexity for an optimal Gaussian kernel perceptron 
for \eqnref{eqn: objectkernel} and $|\textbf{err}(f^*) - \textbf{err}(\hat{f})|$ where $f^*$ is the optimal separator and $\hat{f}$ is the solution to $\cA_{opt}(\mathcal{TS}_{\thetab^*})$ for the constructed teaching set $\mathcal{TS}_{\thetab^*}$.

\subsubsection{Gaussian Kernel Approximation}\label{subsec: gaussian_kernel_approx}
Now, we would talk about finite-dimensional polynomial approximation $\Tilde{\Phi}$ to the Gaussian feature map $\Phi$ via projection as shown in \citet{Cotter2011ExplicitAO}. Consider
\begin{align*}
    \Tilde{\Phi}&: \Rd \longrightarrow \mathbb{R}^q\\
    \Tilde{\rK}(\bx, \bx') &= \Tilde{\Phi}(\bx)\cdot\Tilde{\Phi}(\bx')
\end{align*}
With these approximations, we consider classifiers of the form $\Tilde{f}(\bx) = \Tilde{\thetab}\cdot \Tilde{\Phi}(\bx)$ such that $\Tilde{\thetab} \in \Rq$. Now, assume that there is a projection map $\p$ such that $\Tilde{\Phi} = \p\Phi$.
In \citet{Cotter2011ExplicitAO}, authors used the following approximation to the Gaussian kernel:
\begin{equation}
     \Tilde{\rK}(\bx, \bx') =  \mathbi{e}^{-\frac{||\bx||^2}{2\sigma^2}}\mathbi{e}^{-\frac{||\bx'||^2}{2\sigma^2}}\sum_{k=0}^{s}\frac{1}{k!}\paren{\frac{\normg{\bx}{\bx'}}{\sigma^2}}^k \label{eqn: eqn17}
\end{equation}
This gives the following explicit feature representation for the approximated kernel:
\begin{equation}
\forall k \le s,\quad
\Tilde{\Phi}_{k,\boldsymbol{\lambda}}(\bx) = \Phi_{k,\boldsymbol{\lambda}}(\bx) = \mathbi{e}^{-\frac{||\bx||^2}{2\sigma^2}}\cdot  \frac{\sqrt{\cC^k_{\boldsymbol{\lambda}}}}{\sqrt{k!}\sigma^k}\cdot\bx^{\boldsymbol{\lambda}}   \label{eqn:eqn18}
\end{equation}
where $\Phi_{k,\boldsymbol{\lambda}}(\bx)$ is the coordinate for Gaussian feature map. 
Note that the feature map $\Tilde{\Phi}$ defined by the explicit features in \eqnref{eqn:eqn18} has dimension $\binom{d+s}{d}$. Thus, $\p\Phi = \Tilde{\Phi}$ where the first $\binom{d+s}{d}$ coordinates are retained. We denote the RKHS corresponding to $\Tilde{\rK}$ as $\cH_{\Tilde{\rK}}$.
A simple property of the approximated kernel map is stated in the following lemma which was proven in \citet{Cotter2011ExplicitAO}.
\begin{lemma}[\citet{Cotter2011ExplicitAO}]\label{lemma: approxbound}
For the approximated map $\Tilde{\rK}$, we obtain the following upper bound:
\begin{equation}
    \left|\rK(\bx,\bx) - \Tilde{\rK}(\bx,\bx)\right| \le \frac{1}{(s+1)!}\paren{\frac{||\bx||\cdot ||\bx'||}{\sigma^2}}^{s+1} \label{eqn: approximatekernel}
\end{equation}
\end{lemma}
Note that if $s$ is chosen large enough and the points $\bx, \bx'$ are bounded wrt $\sigma^2$, then  RHS of \eqnref{eqn: approximatekernel} can be bounded by any $\epsilon > 0$. Since $\left|\rK(\bx,\bx) - \Tilde{\rK}(\bx,\bx)\right| = \inmod{\p^{\bot}\Phi(\bx)}^2$, thus for a Gaussian kernel, information theoretically, the first $\binom{d+s}{s}$ coordinates are highly sensitive. We would try to analyze this observation under some mild assumptions on the data distribution to construct an $\epsilon$-approximate teaching set. As discussed in \iftoggle{longversion}{\appref{appendix: gaussian perceptron}}{ the supplemental materials}, we would find the value of $s$ as if the datapoints are coming from a ball of radius $R := \max\left\{\frac{\log^2 \frac{1}{\epsilon}}{e^2},d\right\}$ in $\Rd$ i.e. $\frac{\inmod{\bx}^2}{\sigma^2} \le R$. Thus, we wish to solve for the value of $s$ such that $\frac{1}{(s+1)!}\cdot \paren{R}^{s+1} \le \epsilon$.

To approximate $s$ we use Sterling's approximation, 
which states that for all positive integers $n$, we have $$\sqrt{2\pi}n^{n+1/2}e^{-n} \le n! \le en^{n+1/2}e^{-n}.$$
Using the bound stated in \lemref{lemma: approxbound}, we fix the value for $s$ as $e^2\cdot R$. We would assume that $R = \frac{\log^2 \frac{1}{\epsilon}}{e^2}$ since we wish to achieve arbitrarily small $\epsilon$-approximate\footnote{When $R = d$ all the key results follow the same analysis.} teaching set.
We define $r:= r(\thetab^*, \epsilon) = \binom{d+s}{s}$.

\vspace{-1mm}
\subsubsection{Bounding the Error}\label{subsec: bounded_error}
\vspace{-1mm}
In this subsection, we discuss our key results on approximate teaching of a Gaussian kernel perceptron learner under some mild assumptions on the target model $\thetab^*$. In order to show $\left|\textbf{err}(f^*) - \textbf{err}(\hat{f})\right| \le \epsilon$ via optimizing to a solution $\hat{\thetab}$ for \eqnref{eqn: objectkernel}, we would achieve a point-wise $\epsilon$-closeness between $f^*$ and $\hat{f}$. Specifically, we show that $\left|f^*(\bx)-\hat{f}(\bx)\right| \le \epsilon$ universally which is similar in spirit to universal approximation theorems~\cite{Liang2017WhyDN, lu2020universal, Yarotsky2017ErrorBF} for neural networks. We prove that this universal approximation could be achieved with $d^{\bigO{\log^2 \frac{1}{\epsilon}}}$ size teaching set.

We assume that the input space $\cX$ is bounded such that $\forall \bx \in \cX$\,\,  $\frac{\norm{\bx}{\bx}}{\sigma^2} \le 2\sqrt{R}$. Since the motivation is to find classifiers which are close to the optimal one point-wise, thus we assume that target model $\thetab^*$ has unit norm.
As mentioned in \eqnref{eqn: kernelfunction}, we can write the target model $\thetab^* \in \cH_{\rK}$ as $\thetab^* = \sum_{i=1}^l \alpha_i\cdot\rK(\ba_i, \cdot)$ for some $\{\ba_i\}_{i=1}^l \subset \cX$ and $\alpha_i \in \reals$. The classifier corresponding to  $\thetab^*$ is represented by $f^*$. \eqnref{eqn: objectkernel} can be rewritten corresponding to a teaching set $\cD := \curlybracket{\paren{\bx_i,\; y_i}}_{i=1}^n$ as:
\begin{equation}
    \cA_{opt} := \argmin_{\beta \in \reals^l} \sum_{i = 1}^n \max\bigparen{-y_i\cdot \sum_{j=1}^l \beta_j\cdot\rK(\ba_j,\bx_i),\; 0}\label{eqn: bounded1}
\end{equation}

Similar to \assref{assumption: polyorthogonal} (cf \secref{subsec.poly}), to construct an approximate teaching set we assume a target model $\thetab^*$ has the property that for some truncated polynomial space $\cH_{\Tilde{\rK}}$ defined by feature map $\Tilde{\Phi}$ there are linearly independent projections in the orthogonal complement of $\p\thetab^*$ in $\cH_{\Tilde{\rK}}$. More formally, we state the property as an assumption which is discussed in details \iftoggle{longversion}{in \appref{appendix: motivation}}{in the supplemental materials}.
\begin{assumption}[Existence of orthogonal classifiers]\label{assumption: orthogonal} For the target model $\thetab^*$ and some $\epsilon > 0$, we assume that there exists $r= r(\thetab^*, \epsilon)$ 
such that $\mathbb{P}\thetab^*$ 
has $r-1$ linear independent projections on the orthogonal subspace of $\mathbb{P}\thetab^*$ in $\cH_{\Tilde{\rK}}$ of the form $\{\Tilde{\Phi}(\bz_i)\}_{i=1}^{r-1}$ such that $\forall i\,\, \bz_i \in \cX $. 
\end{assumption}
For the analysis of the key results, we impose a smoothness condition on the linear independent projections $\{\Tilde{\Phi}(\bz_i)\}_{i=1}^{r-1}$ that they are oriented away by a factor of $\frac{1}{r-1}$. Concretely, for any $i,j$ $\left|\Tilde{\Phi}(\bz_i)\cdot\Tilde{\Phi}(\bz_j)\right| \le \frac{1}{2(r-1)}$. This smoothness condition is discussed in the supplemental. 
Now, we consider the following reformulation of the optimization problem in \eqnref{eqn: bounded1} as follows:
\begin{align}
    \hspace*{-3mm}\cA_{opt} := \argmin_{\beta_0 \in \reals,\, \gamma \in \reals^{r-1}} \sum_{i = 1}^{2r-1} \max\paren{\ell(\beta_0, \gamma, \bx_i, y_i),\; 0}\label{eqn: bounded}
\end{align}
where for any $i \in \bracket{2r-1}$ \vspace{-3mm}$$ \ell(\beta_0, \gamma, \bx_i, y_i) = -y_i\cdot\bigparen{ \beta_0\cdot \rK(\ba,\bx_i) + \sum_{j=1}^{r-1}\gamma_j\cdot \rK(\bz_j, \bx_i)} \nonumber$$ 
and 
with respect to the teaching set
\begin{align}
    \mathcal{TS}_{\thetab^*} := \curlybracket{\parenb{\bz_i, 1},\, \parenb{\bz_i, -1}}_{i = 1}^{r-1} \cup \curlybracket{\parenb{\ba, 1}}\label{eqn: teaching set}
\end{align}
where $\ba$ is chosen such that $\p\thetab^*\cdot\p\Phi(\ba) > 0$\footnote{We assume $\thetab^*$ is non-degenerate in $\Tilde{\rK}$ (as for polynomial kernels in \secref{subsec.poly}) i.e. has points $\ba \in \cX$ such that $\p\thetab^*\cdot\p\Phi(\ba) > 0$ (classified with label 1).} and $\Phi(\ba)\cdot\Phi(\bz_i) \le Q\cdot \epsilon$ (where $Q$ is a constant). $\ba$ could be chosen from a $\mathcal{B}(\sqrt{2\sqrt{R}\sigma^2},0)$ spherical ball in $\reals^d$. We index the set $\mathcal{TS}_{\thetab^*}$ as $\curlybracket{\paren{\bx_i, y_i}}_{i=1}^{2r-1}$. \eqnref{eqn: bounded} is optimized over $\hat{\thetab} = \beta_0\cdot \rK(\ba,\cdot) + \sum_{j=1}^{r-1}\gamma_j\cdot \rK(\bz_j, \cdot)$ such that $\hat{\thetab}\cdot\Phi(\ba) > 0$ and $\{\Phi(\bz_i)\}_{i=1}^{r-1}$ satisfy \assref{assumption: orthogonal} where $\inmod{\hat{\thetab}} = \bigO{1}$.
\begin{figure*}[t!]
\centering
	\begin{subfigure}[b]{0.28\textwidth}
	   \centering
		\includegraphics[width=\linewidth]{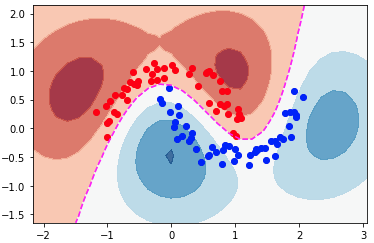}
		\caption{Optimal Gaussian boundary}
		\label{fig:Teacher_RBf}
	\end{subfigure}
	\qquad
	\begin{subfigure}[b]{0.28\textwidth}
	    \centering
		\includegraphics[width=\linewidth]{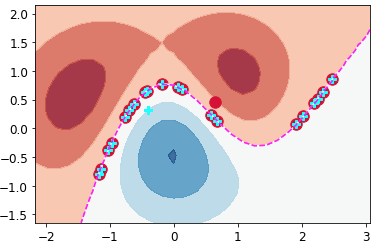}
		\caption{Polynomial approximation}
		\label{fig:Polynomial_approx}
	\end{subfigure}	
	\qquad
	\begin{subfigure}[b]{0.28\textwidth}
	    \centering
		\includegraphics[width=\linewidth]{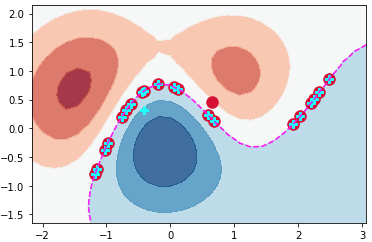}
		\caption{Taught Gaussian boundary}
		\label{fig:Learned_RBF}
	\end{subfigure}	
  \caption{Approximate teaching for Gaussian kernel perceptron. (a) Teacher ``receives'' $\thetab^*$ by training from the complete data set; (b) Teacher identifies a polynomial approximation of the Gaussian decision boundary and generates the teaching set $\mathcal{TS}_{\thetab^*}$ (marked by red dots and cyan crosses); (c) Learner learns a Gaussian kernel perceptron from $\mathcal{TS}_{\thetab^*}$.}
	\label{fig:example-RBF}
\end{figure*}

Note that any solution to $\eqnref{eqn: bounded}$ can have unbounded norm and can extend in arbitrary directions, thus we make an assumption on the learner which would be essential to bound the error of optimal separator of \eqnref{eqn: bounded}.
\begin{assumption}[Bounded Cone]\label{assumption: bounded cone}

For the target model $\thetab^* = \sum_{i = 1}^l \alpha_i\cdot\rK(\ba_i,\cdot)$, the learner optimizes to a solution $\hat{\thetab}$ for \eqnref{eqn: bounded} with bounded coefficients.
Alternatively, the sums $\sum_{i=1}^l \left|\alpha_i\right|$ and $\left|\beta_0\right| + \sum_{j=1}^{r-1}\left|\gamma_j\right|$ are bounded where $\hat{\thetab} \in \cH_{\rK}$ has the form $\hat{\thetab} =  \beta_0\cdot \rK(\ba_j,\cdot) + \sum_{j=1}^{r-1}\gamma_j\cdot \rK(\bz_j, \cdot)$. 


\end{assumption}
This assumption is fairly mild or natural in the sense that for $\hat{\thetab} \in \cA_{opt}$ as a classifier approximates $\thetab^*$ point-wise then they shouldn't be highly (or unboundedly) sensitive to datapoints involved in the classifiers. It is discussed in greater details \iftoggle{longversion}{in \appref{appendix: motivation}}{in the supplemental materials}. We denote by $\boldsymbol{C}_{\epsilon} := \sum_{i=1}^l |\alpha_i|$ and $\boldsymbol{D}_{\epsilon}:= |\beta_0| + \sum_{j=1}^{r-1}|\gamma_j|$. In \iftoggle{longversion}{ \appref{appendixsub: solutionexists}}{ the supplemental materials}, we show that there exists a unique solution (upto a positive scaling) to \eqnref{eqn: bounded} which satisfies \assref{assumption: bounded cone}. We would show that $\mathcal{TS}_{\thetab^*}$ is an $\epsilon$-approximate teaching set with 
$r = d^{\bigO{\log^2 \frac{1}{\epsilon}}}$ on the $\epsilon$-approximate teaching dimension. To achieve this, we first establish the $\epsilon$-\tt{closeness} of $\hat{f}$ (classifier  $\hat{f}(\bx) := \hat{\thetab}\cdot\Phi(\bx)$ where $\hat{\thetab} \in \cA_{opt}$) to $f^*$.
Formally, we state the result as follows:

\begin{theorem}\label{thm: boundedclassifier}
For any target $\thetab^* \in \cH_{\rK}$ that satisfies \assref{assumption: orthogonal}-\ref{assumption: bounded cone}
and $\epsilon > 0$, the teaching set $\mathcal{TS}_{\thetab^*}$ constructed for \eqnref{eqn: bounded} satisfies
    $\left|f^*(\bx) - \hat{f}(\bx)\right| \le \epsilon$
for any $\bx \in \cX$ and any $\hat{f} \in \cA_{opt}(\mathcal{TS}_{\thetab^*})$.
\end{theorem}
Using \thmref{thm: boundedclassifier}, we can obtain the main result of the subsection which gives an $d^{\bigO{\log^2 \frac{1}{\epsilon}}}$ bound on $\epsilon$-approximate teaching dimension. We detail the proofs \iftoggle{longversion}{in \appref{appendix: gaussian perceptron}}{in the supplemental materials}:

\begin{theorem}\label{thm: gaussian_main_thm}
For any target $\thetab^* \in \cH_{\rK}$ that satisfies \assref{assumption: orthogonal}-\ref{assumption: bounded cone}
and $\epsilon > 0$, the teaching set $\mathcal{TS}_{\thetab^*}$ constructed for \eqnref{eqn: bounded}
is an
$\epsilon$-approximate teaching set with 
$\epsilon$-$TD(\thetab^*,\cA_{opt}) = d^{\bigO{\log^2 \frac{1}{\epsilon}}}$ i.e. for any $\hat{f} \in \cA_{opt}(\mathcal{TS}_{\thetab^*})$, 
    $$\left|\textbf{err}(f^*) - \textbf{err}(\hat{f})\right| \le \epsilon.$$
\end{theorem}

\paragraph{Numerical example} 
\figref{fig:example-RBF} demonstrates the approximate teaching process for a Gaussian learner. We aim to teach the optimal model $\thetab^*$ (infinite-dimensional) trained on a pre-collected dataset with Gaussian parameter $\sigma = 0.9$, whose corresponding boundary is shown in \figref{fig:Teacher_RBf}. Now, for approximate teaching, the teacher calculates $\Tilde{\thetab}$ using the polynomial approximated kernel (i.e. $\Tilde{\rK}$, and in this case, k=5) in \eqnref{eqn: eqn17} and the corresponding feature map in \eqnref{eqn:eqn18}. 
To ensure \assref{assumption: orthogonal} is met while generating teaching examples for $\Tilde{\thetab}$, we employ the randomized algorithm (as was used in \secref{subsec.poly}) with the key idea of ensuring that the teaching examples on the boundary are linearly independent in the approximated polynomial feature space, i.e. $\Tilde{\rK}(\bz_i, \bz_j) = 0$. Finally, the Gaussian learner receives $\mathcal{TS}_{\thetab^*}$ and learns the boundary shown in \figref{fig:Learned_RBF}. Note the slight difference between the boundaries in \figref{fig:Polynomial_approx} and in \figref{fig:Learned_RBF} as the learner learns with a Gaussian kernel.



\section{Conclusion}
We have studied and extended the notion of teaching dimension for optimization-based perceptron learner. We also studied a more general notion of approximate teaching which encompasses the notion of exact teaching. To the best of our knowledge, our exact teaching dimension for linear and polynomial perceptron learner is new; so is the upper bound on the approximate teaching dimension of Gaussian perceptron learner and our analysis technique in general. There are many possible extensions to the present work. For example, one may extend our analysis to relaxing the assumptions imposed on the data distribution for polynomial and Gaussian perceptrons. This can potentially be achieved by analysing the linear perceptron and finding ways to nullify subspaces other than orthogonal vectors. This could enhance the results for both the exact teaching of polynomial perceptron learner to more general case and a tighter bound on the approximate teaching dimension of Gaussian perceptron learner. 
On the other hand, a natural extension of our work is 
to understand the approximate teaching complexity for
other types of ERM learners, e.g. kernel SVM, kernel ridge, and kernel logistic regression. We believe 
the current work and its extensions would enrich our understanding of optimal and approximate teaching and enable novel applications.

\section{Acknowledgements}
Yuxin Chen is supported by NSF 2040989 and a C3.ai DTI Research Award 049755.


\bibliographystyle{icml2021}
\bibliography{reference}


\iftoggle{longversion}{
 \newpage
\onecolumn
\appendix
{\allowdisplaybreaks
  \section{List of Appendices}\label{appendix:table-of-contents}
 First, we  provide the proofs of our theoretical results in full detail in the subsequent sections. We follow these by the experimental evaluations section. The appendices are summarized as follows: 
\begin{itemize}
\item \appref{appendix: polynomial perceptron} provides the proof of \thmref{thm: poly_main_theorem}
\item \appref{appendix: motivation} provides the motivations and key insights into \assumref{assumption: polyorthogonal}{assumption: orthogonal}{assumption: bounded cone}
\item \appref{appendix: gaussian perceptron} provides relevant results and proofs of 
\thmref{thm: boundedclassifier} and \thmref{thm: gaussian_main_thm} (Approximate Teaching Set for Gaussian Learner)
\item \appref{appendix: experimentals} provides the experimental evaluations for the theoretical results on various datasets
\end{itemize}

  \section{Polynomial Kernel Perceptron}\label{appendix: polynomial perceptron}
In this appendix, we would provide the proof for the main result of \secref{subsec.poly} i.e. \thmref{thm: poly_main_theorem}. We would complete the proof by constructing a teaching set for exact teaching. Similar to the proof of \thmref{thm: linear perceptron main result}, the key idea is to find linear independent polynomials on the orthogonal subspace defined by $\thetab^* \in \cH_{\rK}$. Our \assref{assumption: polyorthogonal} would ensure that there are such $r-1$ linear independent polynomials. Rest of the work follows steps inspired as seen in the proof of \thmref{thm: linear perceptron main result}. 
We assume that $\thetab^*$ is non-degenerate and has at least one point in $\Rd$ classified \tt{strictly} positive, and provide the poof below. 
\begin{proof}[Proof of \thmref{thm: poly_main_theorem}]
First, we would show the construction of a teaching set for a target model $\thetab^* \in \cH_{\rK}$. Denote by $\mathcal{V}^{\bot}_{\thetab^*} \subset \cH_k$ ($\cong$ $\cH_{\rK}$ using \propref{prop: polynomial space} i.e isomorphic as vector spaces) the orthogonal subspace of $\thetab^*$. Since $\thetab^*$ satisfies \assref{assumption: polyorthogonal}, thus there exists a set of $r-1$ linearly independent vectors (polynomials because of \propref{prop: polynomial space}) of the form $\{\Phi(\bz_i)\}_{i=1}^{r-1}$ in $\mathcal{V}^{\bot}_{\thetab^*}$ where $\curlybracket{\bz_i}_{i=1}^{r-1} \in \Rd$. Note that $\thetab^*\cdot\Phi(\bz_i) = 0$. Now, pick $\ba \in \Rd$ such that $\thetab^*\cdot\Phi(\ba) >0$ (assuming non-degeneracy). We note that $\curlybracket{(\bz_i,1)}_{i=1}^{r-1}\cup \curlybracket{(\bz_i,-1)}_{i=1}^{r-1}\cup \curlybracket{(\ba,1)}$ forms a teaching set for the decision boundary corresponding to $\thetab^*$. Using similar ideas from the proof of \thmref{thm: linear perceptron main result}, we notice that any solution $\hat{\thetab}$ to \eqnref{eqn: objectkernel} satisfies $\thetab^*\cdot\Phi(\bz_i) = 0$ for the labelled datapoints corresponding to $\Phi(\bz_i)$. Thus, $\hat{\thetab}$ doesn't have any component along $\Phi(\bz_i)$. \eqnref{eqn: objectkernel} is minimized if $\hat{\thetab}\cdot\Phi(\ba) \ge 0$ implying $\hat{\thetab} = t\thetab^*$. Thus, under \assref{assumption: polyorthogonal}, we show an upper bound $\bigO{\binom{d+k-1}{k}}$ on the size of a teaching set for $\thetab^*$.
\end{proof}
\newpage
  \section{Motivation for Assumptions}\label{appendix: motivation}
In this appendix, we discuss the motivations and insights for the key \assumref{assumption: polyorthogonal}{assumption: orthogonal}{assumption: bounded cone} made in \secref{subsec.poly} and \secref{subsec.gaussiankernel}. This appendix is organized in the following way: \appref{appsubsec.limitation} discusses \assref{assumption: polyorthogonal} and provides the proofs of \lemref{lemma: exact teaching} and \lemref{lemma: approximate teaching } in the context of polynomial kernel (see \secref{subsec.poly}); \appref{appsubsec.approxassumtions} discusses the \assref{assumption: orthogonal} and \assref{assumption: bounded cone} in the context of Gaussian kernel perceptron (see \secref{subsec.gaussiankernel}).


\paragraph{Reformulation of a model $\thetab$ as a polynomial form} As noted in \secref{sec.statement}, we consider the reproducing kernel Hilbert space~\cite{learnkernel} $\cH_{\rK}$ which could be spanned by the linear combinations of kernel functions of the form $\rK(\bx, \cdot)$. More concretely,
\begin{equation}
    \cH_{\rK} = \condcurlybracket{\sum_{i=1}^m \alpha_i\cdot\rK(\bx_i,\cdot)}{ m \in \mathbb{N},\, \bx_i \in \cX,\, \alpha_i \in \reals, i = 1,\cdots,m}\nonumber
\end{equation}
Thus, we could write any model $f_{\thetab} \in \cH_{\rK}$ (parametrized by $\thetab$) as $\sum_{i=1}^n \alpha_i\cdot\rK(\bx_i,\cdot)$ for some $n \in \mathbb{N}$, $
\bx_i \in \cX$ for $i \in \bracket{n}$. This interesting because if $\rK(\cdot,\cdot)$ is a polynomial kernel of degree $k$, then
\begin{equation}
    f_{\thetab}(\bx) = \sum_{i=1}^n \alpha_i\cdot\rK(\bx_i,\bx)  = \sum_{i=1}^n \alpha_i\cdot \normg{\bx_i}{\bx}^k = \sum_{i=1}^n \alpha_i\cdot \paren{\bx_{i1}\bx_1+\cdots+\bx_{id}\bx_d}^k \label{eqn: poly form}
\end{equation}
where $\bx_i = \paren{\bx_{i1},\cdots,\bx_{id}}$. Thus, $f_{\thetab}(\cdot)$ could be reformulated as a homogeneous polynomial of degree $k$ in $d$ variables. Notice that for polynomial kernel in \secref{subsec.poly}, for a target model $\thetab^*$ we study the orthogonal projections of the form $\Phi(\bx)$ for $\bx \in \cX$ such that $\thetab^*\cdot\Phi(\bx) = 0$. Alternatively, using \eqnref{eqn: poly form} we wish to solve the polynomial equation:
\begin{equation}
    f_{\thetab^*}(\bx) = 0 \implies \sum_{i=1}^n \alpha_i\cdot \paren{\bx_{i1}\bx_1+\cdots+\bx_{id}\bx_d}^k = 0 \nonumber
\end{equation}
where we denote $\thetab^* := \sum_{i=1}^n \alpha_i\cdot\Phi(\bx_i)$. For \assref{assumption: polyorthogonal} we wish to find $\binom{d+k-1}{k}$ real solutions of the form $\bx' \in \reals^d$, of this equation which are linearly independent. It is well-studied in the literature of polynomial algebra that this equation might not satisfy the required assumption. We construct one such model for the proof of \lemref{lemma: approximate teaching }. This reformulation can be extended for sum of polynomial kernels of the form $\sum_{j=1}^s c_j\cdot{\normg{\bx}{\bx'}}^j$ where $c_j \ge 0$. In \assref{assumption: orthogonal} the reformulation reduces to a variant of the above polynomial equation i.e.
\begin{equation*}
    \sum_{i=1}^n\alpha_i\paren{\sum_{j=1}^s c_i\cdot \paren{\bx_{i1}\bx_1+\cdots+\bx_{id}\bx_d}^j} = 0
\end{equation*}
So far, we discussed a characterization of the notion of orthogonality for a target model $\thetab^*$ in the form of a polynomial equation. This characterization would help us understand \assref{assumption: polyorthogonal} and \assref{assumption: orthogonal}. In \secref{appsubsec.limitation}, we discuss that \assref{assumption: polyorthogonal} is the most natural step to make for exact teaching of a target model. 

\subsection{ Limitation of Exact Teaching: Polynomial Kernel Perceptron}\label{appsubsec.limitation}
In this subsection, we provide the proofs of \lemref{lemma: exact teaching} and \lemref{lemma: approximate teaching } as stated in \secref{subsec.motivation}. These results establish that in the realizable setting, \assref{assumption: polyorthogonal} is required for exact teaching: \lemref{lemma: exact teaching}. Furthermore, there are pathological cases where violation of the assumption leads to models which couldn't be approximately taught: \lemref{lemma: approximate teaching }.

\begin{proof}[Proof of \lemref{lemma: exact teaching}]
We would prove the result by contradiction. Assume that $\mathcal{TS}_{\thetab^*}$ be a teaching set which \tt{exactly} teaches $\thetab^*$. $\mathrm{WLOG}$ we enumerate the teaching set as $\mathcal{TS}_{\thetab^*} = \curlybracket{\paren{\bx_1,y_1},\cdots,\paren{\bx_n,y_n}}$. For the sake of clarity, we would rewrite \eqref{eqn: objectkernel} again
\begin{equation}
    \cA_{opt}(\mathcal{TS}_{\thetab^*}):= \argmin_{\thetab \in \cH_{\rK}}\sum_{i=1}^n \max(-y_i\cdot \thetab^*\cdot\Phi(\bx_i), 0)\label{eq: polyobject}
\end{equation}
 Denote by $\mathcal{V}^{\bot}_{\thetab^*} \subset \cH_k$  the orthogonal subspace of $\thetab^*$. We denote the objective value of \eqnref{eq: polyobject} by $\bp(\thetab) := \sum_{i = 1}^{n} \max(-y_i\cdot\thetab\cdot\Phi(\bx_i),\: 0)$. We further define \tt{effective direction of a teaching point} $\paren{\bx_i,y_i} \in \mathcal{TS}_{\thetab^*}$ in the RKHS $\cH_{\rK}$ as $\dd_i := -y_i\cdot\Phi(\bx_i)$.  
 Because of the realizable setting i.e. all teaching points are correctly classified, it is clear that  
$$-y_i\cdot \thetab^*\cdot\Phi(\bx_i) \le 0 \implies \thetab^*\cdot \dd_i \le 0.$$
Since $\thetab^*$ violates \assref{assumption: polyorthogonal}, thus $\exists$ a unit normalized direction $\hat{\dd} \in \mathcal{V}^{\bot}_{\thetab^*}$ which can't be spanned by $\cS_{0} \triangleq \condcurlybracket{\Phi(\bx)}{ \Phi(\bx) \in \mathcal{V}^{\bot}_{\thetab^*}\,\, \textnormal{for some}\,\, \bx \in \cX}$ such that $\hat{\dd} \perp \mathbf{span}\left\langle\cS_{0}\right\rangle$. Now, we would show that $\exists \lambda > 0$ (real) such that 
\begin{equation}
  \paren{\thetab^* + \lambda\hat{\dd}} \in \cA_{opt}(\mathcal{TS}_{\thetab^*})  \label{eqn: new theta}
\end{equation}
Notice that for some $i$ if $\dd_i \in \mathcal{V}^{\bot}_{\thetab^*}$ then $(\thetab^* + \lambda\hat{\dd})\cdot \dd_i = \thetab^*\cdot \dd_i \le 0$. Now, we consider the case when $\dd_i \notin \mathcal{V}^{\bot}_{\thetab^*}$. We could expand $\dd_i$ as follows:
\begin{equation}
\dd_i = a_i\hat{\dd}^{\perp} + b_i\hat{\dd} \label{eqn: expand}    
\end{equation}
where $a_i$ and $b_i$ are real scalars and $\hat{\dd}^{\perp}$ is normalized orthogonal projection of $\dd_i$ to orthogonal complement (orthogonal subspace) of $\hat{\dd}$. These constructions are illustrated in \figref{fig:lemma1}. Now, we would compute the following dot product:
\begin{align}
    &(\thetab^* + \lambda\hat{\dd})\cdot \dd_i \nonumber\\
\implies& \thetab^*\cdot \dd_i + \lambda\hat{\dd}\cdot (a_i\hat{\dd}^{\perp} + b_i\hat{\dd}) \label{eqn: part1}\\
\implies& \thetab^*\cdot\dd_i + \lambda\hat{\dd}\cdot b_i\hat{\dd} \label{eqn: part2}
\end{align}
\eqnref{eqn: part1} follows using \eqnref{eqn: expand}. In \eqnref{eqn: part2} we note that $\hat{\dd} \perp \hat{\dd}^{\perp}$. If $b_i \le 0$ then $(\thetab^* + \lambda\hat{\dd})\cdot \dd_i \le 0$ as $\thetab^*\cdot \dd_i < 0$. If $b_i > 0$, then to ensure $(\thetab^* + \lambda\hat{\dd})\cdot \dd_i \le 0$, using \eqnref{eqn: part2} we need
$$\lambda \le \frac{-\thetab^*\cdot \dd_i}{b_i}$$
Since, $i$ is chosen arbitrarily thus for all the effective directions $\dd_i \notin \mathcal{V}^{\bot}_{\thetab^*}$  where $b_i > 0$, we pick postive scalar $\lambda$ such that: 
$$\lambda := \min_{i:\, b_i > 0}\frac{-\thetab^*\cdot \dd_i}{b_i}$$
For this choice of $\lambda$ we show that $\thetab^* + \lambda\hat{\dd} \in \cA_{opt}(\mathcal{TS}_{\thetab^*})$. Thus, by definition, $\mathcal{TS}_{\thetab^*}$ as stated above can't teach $\thetab^*$ exactly. Hence, if $\thetab^*$ violates \assref{assumption: polyorthogonal} then we can't teach it exactly in the realizable setting.
\end{proof}

\begin{figure*}[t]
\centering
    \begin{subfigure}[b]{0.30\textwidth}
	    \centering
		\includegraphics[width=\linewidth]{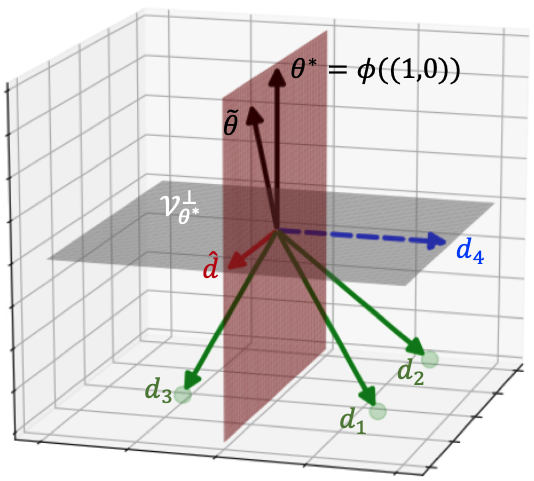}
		\caption{}
		\label{fig:lemma1}
	\end{subfigure}	
	\qquad
	\begin{subfigure}[b]{0.40\textwidth}
	    \centering
		\includegraphics[width=\linewidth]{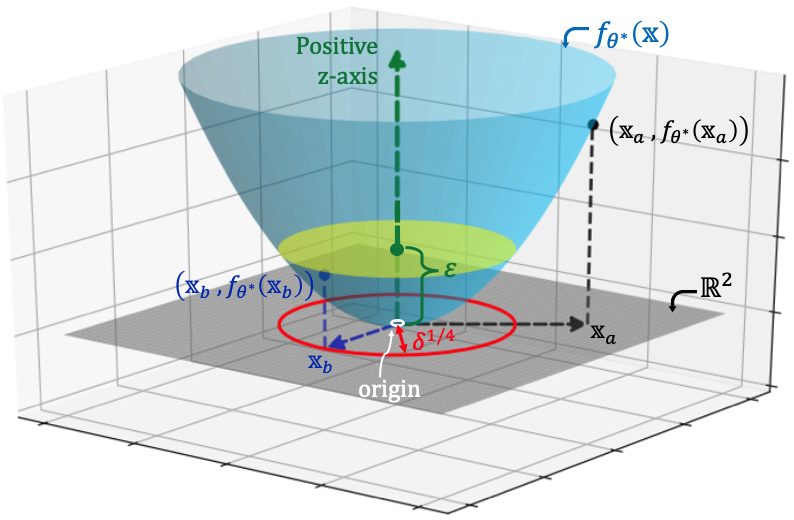}
		\caption{}
		\label{fig:lemma2}
	\end{subfigure}	
  \caption{Illustrations for Proofs of \lemref{lemma: exact teaching} and \lemref{lemma: approximate teaching }. (a) For \lemref{lemma: exact teaching}, consider $\thetab^*$ as shown. $\mathbf{span}\left\langle\cS_{0}\right\rangle$ only covers the direction indicated by the dashed blue arrow. $\Tilde{\thetab}$ correctly labels not only teaching examples on $\mathcal{V}^{\bot}_{\thetab^*}$ and within $\mathbf{span}\left\langle\cS_{0}\right\rangle$, but also those not on $\mathcal{V}^{\bot}_{\thetab^*}$, e.g. along effective directions $d_1$, $d_2$, $d_3$; (b) To visualize the proof idea of \lemref{lemma: approximate teaching }, we demonstrate an example in $\reals^2$ with feature space of dimension 3 (where $k = 2$). We consider a model $\thetab^* = \frac{1}{\sqrt{2}}\cdot\Phi((1,0))+\frac{1}{\sqrt{2}}\cdot\Phi((0,1))$. Since $k$ is even, each point $\textbf{x}$ in $\mathbb{R}^2$ corresponds to a non-negative value of  $f_{\thetab^*}(\textbf{x})$. This function is plotted as the blue surface in (b) along $z$-axis. Yellow surface represents a threshold of $\epsilon$ along $z$-axis; thus any point above it has value more than $\epsilon$. A red $\delta$-norm ring on the $\reals^2$-plane denotes the constraint on the norm of a teaching point ($||\Phi(\bx)|| = \delta \implies ||\bx|| = \delta^{1/4}$). If we are constrained to select only points outside of the red $\delta$-norm ring, then the plot illustrates a situation where no points outside the ring satisfies $f_{\thetab^*}(\textbf{x}) = \thetab^*\cdot\Phi(\bx) < \epsilon$.}
	\label{fig:lemma1-2}
\end{figure*}

Now, we provide the proof of \lemref{lemma: approximate teaching } for which we give a construction of a model $\thetab^*$ which violates \assref{assumption: polyorthogonal} and also show that it can't be taught arbitrarily $\epsilon$-close approximately. The proof is also illustrated in \figref{fig:lemma2}.

\begin{proof}[Proof of \lemref{lemma: approximate teaching }] Assume $\thetab^*$ is a target model which violates \assref{assumption: polyorthogonal}. If $\thetab^*$ can be taught \tt{approximately} for arbitrarily small $\epsilon > 0$ then $\exists$ $\Tilde{\thetab}^*$ which can be taught \tt{exactly} (i.e. satisfies \assref{assumption: polyorthogonal}) such that
$$\thetab^*\cdot \Tilde{\thetab}^* \ge 1-\cos{a_{\epsilon}},\,\, \textnormal{where}\,\, \cos{a_{\epsilon}} = \epsilon$$
if $\thetab^*$ and $\Tilde{\thetab}^*$ are unit normalized. This implies that if $\Phi(\bx) \in \mathcal{V}_{\Tilde{\thetab}^*}^{\perp} \subset \cH_{\rK}$ (orthogonal complement of $\Tilde{\thetab}^*$) such that $||\Phi(\bx)|| \le 1$ then the following holds:
\begin{equation}
  |\thetab^*\cdot \Phi(\bx)| \le \epsilon  \label{eqn: almost orth}
\end{equation}
Alternatively, we could think of $\Phi(\bx)$ as being almost orthogonal to $\thetab^*$.

Now, we would show a construction of a target model when $k$ has parity even, which not only violates \assref{assumption: polyorthogonal} but can't be taught \tt{approximately} such that \eqnref{eqn: almost orth} holds. The idea is to find $\thetab^*$ which doesn't have almost orthogonal projections in $\cH_{\rK}$ with norm lower-bounded by $\delta$.

Consider the following construction for a target model $\thetab^* \in \cH_{\rK}$:
\begin{equation}
    \thetab^* = \sum_{i=1}^d \frac{1}{\sqrt{d}}\cdot \Phi(\be_i),\quad ||\thetab^*|| = 1 \label{eqn: thetaconstruction1}
\end{equation}
where $\{\be_i\}$'s form the standard basis in $\reals^d$. Notice that for any $\bx \in \cX$, 
\begin{equation}
  \thetab^*\cdot \Phi(\bx) = \sum_{i=1}^d \frac{1}{\sqrt{d}}\cdot \Phi(\be_i)\cdot\Phi(\bx) = \sum_{i=1}^d \frac{1}{\sqrt{d}}\cdot \bx_i^k.  \label{eqn: thetaconstruction2}
\end{equation}
RHS of the above equation is zero only when all the $\bx_i^k = 0$ since $k$ is even. Thus, the only projection orthogonal to $\thetab^*$ is the zero projection in $\cH_{\rK}$, thus violates \assref{assumption: polyorthogonal}. Now, we show that $\thetab^*$ as constructed in \eqnref{eqn: thetaconstruction1} can't be taught approximately for arbitrarily small $\epsilon > 0$.

If $\bx \in \cX$ is such that $||\Phi(\bx)|| \ge \delta$, then using H\"{o}lder's inequality:
\begin{align}
    \sum_{i=1}^d \bx_i^2 \le \sum_{i=1}^d 1\cdot \bx_i^2 \le d^{\frac{k-2}{k}}\paren{\sum_{i=1}^d \paren{\bx_i^2}^{\frac{k}{2}}}^{\frac{2}{k}} \nonumber
\end{align}
But we have $\sum_{i=1}^d \bx_i^2 \ge \delta^{\frac{2}{k}}$. Thus, using \eqnref{eqn: thetaconstruction1}-\eqnref{eqn: thetaconstruction2}
\begin{equation}
    \paren{\sum_{i=1}^d \bx_i^k} \ge \frac{\delta}{d^{2(k-2)}} \implies \thetab^*\cdot\Phi(\bx) \ge \frac{\delta}{d^{2k-\frac{7}{2}}} \nonumber
\end{equation}
Thus, if $\epsilon < \frac{\delta}{d^{2k-\frac{7}{2}}}$ then $\thetab^*\cdot\Phi(\bx) > \epsilon$. This implies that $\Phi(\bx)$ can't be chosen almost orthogonal to $\thetab^*$ violating \eqnref{eqn: almost orth}. Hence, $\nexists$ $\Tilde{\thetab}^*$ arbitrarily close to $\thetab^*$ which can be taught exactly.

Thus, the construction of $\thetab^*$ in \eqnref{eqn: thetaconstruction1} violates \assref{assumption: polyorthogonal} and can't be taught approximately for arbitrarily small $\epsilon > 0$.
\end{proof}
\paragraph{Is the assumption of lower bound $\delta$ restrictive?} Now, we would argue that the assumption of a lower bound on the norm of the teaching point for \lemref{lemma: approximate teaching } is only for analysis of the proof presented above. Consider the target model $\thetab^*$ constructed in \eqnref{eqn: thetaconstruction1}. Consider that $\exists$ $\Tilde{\thetab}^*$ which can be taught exactly using arbitrarily small normed teaching points (i.e. lower bound of $\delta$ is violated) such that
$$\thetab^*\cdot \Tilde{\thetab}^* \ge 1-\cos{a_{\epsilon}},\,\, \textnormal{where}\,\, \cos{a_{\epsilon}} = \epsilon$$ for arbitrarily small $\epsilon > 0$. Define the teaching set as $\mathcal{TS}_{\Tilde{\thetab}^*}$. But, even if we unit-normalize all the teaching points, call the normalized set $\mathcal{TS}^{unit}_{\Tilde{\thetab}^*}$, \eqnref{eqn: objectkernel} is still satisfied. Since in that case for any $(\bx_i,y_i) \in \mathcal{TS}^{unit}_{\Tilde{\thetab}^*}$, \eqnref{eqn: almost orth} is violated. Hence, violating the assumption of lower bound on the norm of the teaching points doesn't invalidate the claim of \lemref{lemma: approximate teaching }.

\subsection{Approximate Teaching: \assref{assumption: orthogonal} and \assref{assumption: bounded cone}}\label{appsubsec.approxassumtions}
As noted in \secref{subsec.gaussiankernel}, the teaching dimension of a Gaussian kernel perceptron learner is $\infty$. This calls for studying these non-linear kernel in the setting of approximate teaching. Inspired by our discussion in the previous subsection, we argue that the underlying assumptions: \assref{assumption: orthogonal} and \assref{assumption: bounded cone} are fairly mild in order to establish strong results stated in \thmref{thm: boundedclassifier} and \thmref{thm: gaussian_main_thm} (cf. \secref{subsec.gaussiankernel}). This appendix subsection is divided into two paragraphs corresponding to the assumptions as follows:

\paragraph{Existence of orthogonal linear independent projections: \assref{assumption: orthogonal}.} Notice that the projected polynomial space or the approximated kernel $\Tilde{\rK}$ is a sum of polynomial kernels. We rewrite \eqnref{eqn: eqn17} for ease of clarity:
\begin{equation*}
     \Tilde{\rK}(\bx, \bx') =  \mathbi{e}^{-\frac{||\bx||^2}{2\sigma^2}}\mathbi{e}^{-\frac{||\bx'||^2}{2\sigma^2}}\sum_{k=0}^{s}\frac{1}{k!}\paren{\frac{\normg{\bx}{\bx'}}{\sigma^2}}^k 
\end{equation*}
If we replace $z = \frac{\normg{\bx}{\bx'}}{\sigma^2}$, we could write 
\begin{equation*}
    \Tilde{\rK}(\bx, \bx') =  \mathbi{e}^{-\frac{||\bx||^2}{2\sigma^2}}\mathbi{e}^{-\frac{||\bx'||^2}{2\sigma^2}}\sum_{k=0}^{s}\frac{1}{k!}\cdot z^k
\end{equation*}
Since all the coefficients of the polynomial $\sum_{k=0}^{s}\frac{1}{k!}\cdot z^k$ are positive thus if $s$ is even then $\Tilde{\rK}(\bx, \bx') > 0$. Thus, if $\thetab^* = \sum_{i=1}^l \alpha_i\cdot\rK(\ba_i, \cdot)$ for some $\{\ba_i\}_{i=1}^l \subset \cX$ such that $\alpha_i$'s are positive then \assref{assumption: orthogonal} would be violated. Hence, there is a class of $r$ values for which the assumption would be violated.

It is straight-forward to note that \lemref{lemma: exact teaching} could be extended to sum of polynomial kernels. Similar extension for \lemref{lemma: approximate teaching } when the highest degree is of parity even could be established. These results follow by noting the polynomial space of homogeneous polynomials of degree $k$ in $d$ variables is isomorphic \cite{article} to the polynomial space of degree $k$ in $(d-1)$ variables. Since Hilbert space $\Tilde{\rK}$ is a sum of polynomial kernels thus the extended results hold. This implies that there could be pathological cases where $\p\thetab^*$ could not be learnt approximately in $\Tilde{\rK}$. But this poses a problem because most of the information of a model in terms of the eigenvalues of the orthogonal basis of the Gaussian kernel is contained in the starting indices i.e. $\forall k \le s,\quad
\Phi_{k,\boldsymbol{\lambda}}(\bx) = \mathbi{e}^{-\frac{||\bx||^2}{2\sigma^2}}\cdot  \frac{\sqrt{\cC^k_{\boldsymbol{\lambda}}}}{\sqrt{k!}\sigma^k}\cdot\bx^{\boldsymbol{\lambda}} $ where $\sum_{i=1}^d\boldsymbol{\lambda}_i = k$. It has been discussed in \appref{appendix: gaussian perceptron}. Since the fixed Hilbert space induced by $\Tilde{\rK}$ is spanned by these truncated projections, thus \assref{assumption: orthogonal} gives a characterization for approximately teachable models. It is left to be understood if there is a more unified characterization which could incorporate approximately teachable models beyond \assref{assumption: orthogonal}.

\paragraph{Boundedness of weights: \assref{assumption: bounded cone}.} It is fairly natural in the sense that in \thmref{thm: boundedclassifier} we are bounding (approximating) the error values of the function point-wise i.e. $f^*$ (using $\hat{f}$) for a fixed $\epsilon$. If for some $\hat{\boldsymbol{\theta}} \in \mathcal{A}_{opt}$, $\hat{f}$ ( $= \hat{\boldsymbol{\theta}}\cdot\Phi(\cdot)$) is unboundedly sensitive to some teaching (training) point, then bounding error becomes stringent. Further, we show that there exists a unique solution up to a positive constant scaling to \eqnref{eqn: bounded} which satisfy the assumption in \appref{appendixsub: solutionexists}.

The weights $\{\alpha_i\}_{i=1}^l$ and $[\beta_0,\gamma]$ could be thought of as Lagrange multipliers for Gaussian kernel perceptron. Boundedness of the multipliers is a well-studied problem in the mathematical programming and optimization literature  \cite{gauvin,nooshin,dutta,locallipschitz}. Interestingly, \citet{luksan} demonstrated the importance
of the boundedness of the Lagrange multipliers for the study of interior point methods for non-linear programming. On the other hand, \assref{assumption: bounded cone} as a regularity condition provides new insights into solving problems where task is to universally approximate the underlying functions as discussed in the proof of \thmref{thm: boundedclassifier} in \appref{appendixsub: proofofmainthm}.
\newpage

  \section{Gaussian Kernel Perceptron}\label{appendix: gaussian perceptron}
In this appendix, we would provide the proofs to the key results: \thmref{thm: boundedclassifier} and \thmref{thm: gaussian_main_thm}, as shown in \secref{subsec: bounded_error}. The key to establishing the results is to provide a constructive procedure for an approximate teaching set. Under the \assref{assumption: orthogonal} and \assref{assumption: bounded cone}, when the Gaussian learner optimizes \eqnref{eqn: bounded} w.r.t the teaching set, any solution $\hat{\thetab} \in \cA_{opt}$ would be $\epsilon$-close to the optimal classifier point-wise, thereby bounding the error on the data distribution $\cP$ on the input space $\cX$. We organize this appendix as follows: in \appref{appendixsub: solutionexists} we show that there exists a solution to \eqnref{eqn: bounded}; in \appref{appendixsub: proofofmainthm} we provide the proofs for our key results \thmref{thm: boundedclassifier} and \thmref{thm: gaussian_main_thm}.

\paragraph{Truncating the Taylor features of Gaussian kernel.} In \secref{subsec: gaussian_kernel_approx}, we showed the construction of the projection $\p$ such that $\p\Phi$ forms a feature map for the kernel $\Tilde{\rK}$. We denote the orthogonal projection to $\p$ by $\p^{\bot}$. Thus, we can write $\Phi(\bx) = \p\Phi(\bx) + \p^{\bot}\Phi(\bx)$ for any $\bx \in \Rd$. We discussed the choice of $R$ and $s$. The primary motivation to pick them in the certain way is to retain maximum information in the first $\binom{d+s}{s}$ coordinates of $\Phi(\cdot)$. This is in line with the observation that the eigenvalues of the canonical orthogonal basis~\cite{article} (also eigenvectors) for the Gaussian reproducing kernel Hilbert space $\cH_{\rK}$ decays with higher-indexed coordinates, thus the more sensitive eigenvalues are in the first $\binom{d+s}{s}$ coordinates.
Thus, if we could show that $\p\hat{\thetab}$ is $\epsilon$-approximately close
to $\p\thetab^*$ where $\hat{\thetab} \in \cA_{opt}$ is a solution to \eqnref{eqn: bounded}, then $\hat{\thetab}$ also would be $\epsilon$-approximately close to $\thetab^*$. 

\paragraph{What should be an optimal $R$ vs. choice of the index $s$?} In \secref{subsec: gaussian_kernel_approx}, we solved for $s$ such that 
$$\frac{1}{(s+1)!}\cdot \paren{R}^{s+1} \le \epsilon$$
If $\bx \in \mathcal{B}(\sqrt{2\sqrt{R}\sigma^2}, 0)$ then using \lemref{lemma: approxbound} we have 
\begin{equation}
  \inmod{\p^{\perp}\Phi(\bx)}^2 \le \frac{1}{(s+1)!}\cdot \paren{\sqrt{R}}^{s+1} \le \frac{\epsilon}{\paren{\sqrt{R}}^{s+1}} \le \frac{\epsilon}{\paren{\sqrt{d}}^{s}}  \label{eqn: largeeps}
\end{equation}
where the last inequality follows as $R := \max\left\{\frac{\log^2 \frac{1}{\epsilon}}{e^2},d\right\}$. We define $\epsilon_s := \frac{\epsilon}{\paren{\sqrt{d}}^{s}}$. Note that \eqnref{eqn: largeeps} holds for all $\bx \in \cX$\, since   $\frac{\norm{\bx}{\bx}}{\sigma^2} \le 2\sqrt{R}$. This factor $\paren{\sqrt{d}}^{s}$ in the denominator of $\epsilon_s$ would be useful in nullifying any $\sqrt{r}$ term since $r = \binom{d+s}{s} = \bigO{d^{s}}$.  

\subsection{Construction of a Solution to \eqnref{eqn: bounded}}\label{appendixsub: solutionexists}
In this subsection, we would show that $\eqnref{eqn: bounded}$ has a minimizer $\hat{\thetab} \in \cA_{opt}$ such that $\mathbf{p}(\hat{\thetab}) = 0$ where $\mathbf{p}(\cdot)$ is the objective value. 
Notice that for any $i$ the teaching points $\curlybracket{(\bz_i, 1), (\bz_i, -1)}$ are correctly classified only if $\hat{\thetab}\cdot\Phi(\bz_i) = 0$ and $\hat{\thetab}\cdot\Phi(\ba) > 0$.
We define the set $\boldsymbol{\mathrm{B}} = \curlybracket{\bb_1,\bb_2,\cdots,\bb_{r}}$ to represent $\{\bz_i\}_{i=1}^{r-1}\cup \{\ba\}$ in that order.
We define the Gaussian kernel Gram matrix $\boldlam$ corresponding to $\boldsymbol{\mathrm{B}}$ as follows:
\begin{equation}
    \boldlam[i,j] = \rK(\bb_i,\bb_j)\quad \forall i,j \in \bracket{r}
\end{equation}
Since $\{\bz_i\}_{i=1}^{r-1}$ and $\ba$ could be chosen from $\mathcal{B}(\sqrt{2\sqrt{R}\sigma^2},0)$ as for any two points $\bx, \bx'$ in the teaching set $\frac{\inmod{\bx - \bx'}^2}{2\sigma^2} = \bigTheta{\log \frac{1}{\epsilon}}$ thus all the non-diagonal entries of $\boldlam$ could be bounded as $\bigTheta{\epsilon}$. 
Thus, the non-diagonal entries of $\boldlam$ are upper bounded w.r.t to the choice of $\epsilon$. 
We denote the concatenated vector of $\gamma$ and $\beta_0$ by $\boldeta$ as $\boldeta := (\gamma^\top,\beta_0)^\top$. Consider the following matrix equation:
\begin{equation}
    \boldlam\cdot \boldeta = \parenb{\underbrace{0,\cdots,0}_{\textnormal{For}\; \bz_i's},\underbrace{\vphantom{0,\cdots,0}1}_{\textnormal{For}\; \ba}}^{\top}\label{eqn: satisfyobjective}
\end{equation}
Notice that any solution $\boldeta$ to \eqnref{eqn: satisfyobjective} has zero objective value for \eqnref{eqn: bounded}. Since $\sum_{i=1}^r \paren{\boldlam[i,r]\cdot\boldeta_i} = \hat{\thetab}\cdot\Phi(\ba) > 0$ thus we scale the last component of \eqnref{eqn: satisfyobjective} to 1. First, we observe that \eqnref{eqn: satisfyobjective} has a solution because Gaussian kernel Gram matrix $\boldlam$ to the finite set of points is \tt{strictly positive definite} implying $\boldlam$ is invertible. Thus, there is a unique solution $\boldeta_0 \in \reals^{r}$ such that:
\begin{equation}
    \boldlam\cdot \boldeta_0 = \nu^{\top}\nonumber
\end{equation}
where $\nu := (0,0,\cdots,1)$ as shown in \eqnref{eqn: satisfyobjective}. Also,  $\boldeta_0^{\top}\cdot\boldlam\cdot\boldeta_0 = \beta_0 > 0$. Now, we need to ensure that $\boldeta_0$ satisfies \assref{assumption: bounded cone}. To analyse the boundedness, we rewrite the above equation as:
\begin{equation}
    \boldeta_0 = \boldlam^{-1}\cdot \nu^{\top}\nonumber
\end{equation}
To evaluate the entries of $\boldeta_0$, we only need to understand the last column of $\boldlam^{-1}$ (since $\boldlam^{-1}\cdot \nu^{\top}$ contains entries from the last column of $\boldlam^{-1}$). Using the construction of the inverse using the minors of $\boldlam$, we note that $\boldlam^{-1}[i,r] = \frac{1}{\det(\boldlam)}\cdot M_{(i,r)}$, where $M_{(i,r)}$ is the minor of $\boldlam$ corresponding to entry indexed as $(i, r)$ (determinant of the submatrix of $\boldlam$ formed by removing the \tt{i}th row and \tt{r}th column). Note, determinant is an alternating form, which for a square matrix $T$ of dimension $n$ has the explicit sum
$\sum_{\sigma \in S^n}sign(\sigma)\cdot\prod_{i}^n u_{i,\sigma(i)}$, where $T[i,j] = u_{i,j}$. Since the non-diagonal entries of $\boldlam$ are bounded by $\epsilon$, thus we can bound the minors. Note, $|M_{r,r}| \ge 1  - \bigO{\epsilon^2} + \cdots + (-1)^{r-1} \bigO{\epsilon^{r-1}}$ and for $i \neq r$ $|M_{i,r}| \le \bigO{\epsilon} + \bigO{\epsilon^2} + \cdots + \bigO{\epsilon^{r-1}}$. Since, $\epsilon$ is sufficiently small, thus $|M_{r,r}|$ majorizes over  $|M_{i,r}|$ for $i \neq r$. But then $\boldeta_0 = \frac{1}{\det(\boldlam)}\paren{(-1)^{1+r}M_{1, r}, (-1)^{2+r}M_{2, r}, \cdots, (-1)^{r+r}M_{r, r}}^{\top}$. When we normalize $\thetab$, we get $\bar{\boldeta}_0 = \boldeta_0/\inmod{\thetab}$. We note that, $\inmod{\thetab}^2 = \boldeta_0^{\top}\cdot\boldlam\cdot\boldeta_0 = \beta_0$, implying  $\bar{\boldeta}_0 = \boldeta_0/\sqrt{\beta_0}$. Since, $\beta_0 = \frac{1}{\det(\boldlam)}\cdot M_{r,r}$, thus entries of $\bar{\boldeta}_0$ satisfy \assref{assumption: bounded cone}. Thus, we have a solution to \eqnref{eqn: bounded} which satisfies \assref{assumption: bounded cone}. 

\subsection{Proof of \thmref{thm: boundedclassifier} and \thmref{thm: gaussian_main_thm}}\label{appendixsub: proofofmainthm}
In this section, we would establish our key results for the approximate teaching of a Gaussian kernel perceptron. Under the \assref{assumption: orthogonal} and \assref{assumption: bounded cone}, we would show to teach a target model $\thetab^*$ $\epsilon$-approximately we only require at most $d^{\bigO{\log^2\frac{1}{\epsilon}}}$ labelled teaching points from $\cX$. In order to achieve the $\epsilon$-approximate teaching set, we would show that the teaching set $\mathcal{TS}_{\thetab^*}$ as constructed in \eqnref{eqn: teaching set} achieves an $\epsilon$-closeness between $f^* = \thetab^*\cdot\Phi(\cdot)$ and $\hat{f} = \hat{\thetab}\cdot\Phi(\cdot)$ i.e $ \left|f^*(\bx) - \hat{f}(\bx)\right| \le \epsilon$ point-wise.

Before we move to the proofs of the key results, we state the following relevant lemma which bounds the length (norm) of a vector spanned by a basis with the smoothness condition on the basis as mentioned in \secref{subsec: bounded_error}. 

\begin{lemma}\label{lemma: maximum norm}
Consider the Euclidian space $\reals^n$. Assume $\curlybracket{\bv_i}_{i=1}^n$ forms a basis with unit norms. Additionally, for any $i,j$ $\left|\bv_i\cdot \bv_j\right| \le \cos{\theta_0}$ where $ \cos{\theta_0} \le \frac{1}{2n}$.
Fix a small real scalar $\epsilon > 0$. Now, consider any random vector $\bp \in \reals^n$ such that $\forall i \in \bracket{n}$ $|\bp \cdot \bv_i| \le \epsilon$. Then the following bound on $\bp$ holds: 
$$||\bp||_2 \le \sqrt{2n}\cdot \epsilon.$$
\end{lemma}
\begin{proof}
We define $M = \reals^n$ as the space in which $\bp$ and $\bv_i$'s are embedded. Consider another copy of the space $N = \reals^n$ with standard orthogonal basis $\curlybracket{e_1,\cdots,e_n}$. We define the map $\boldsymbol{\mathrm{W}}: M \longrightarrow N$ as follows:
\begin{align*}
    \boldsymbol{\mathrm{W}} &: M \longrightarrow N\\
    q &\mapsto (\bv_1\cdot q,\bv_2\cdot q,\cdots,\bv_n\cdot q)
\end{align*}
Since $\curlybracket{\bv_i}_{i=1}^n$ forms a basis, thus $\boldsymbol{\mathrm{W}}$ is invertible. To ease the analysis, we could assume $\epsilon = 1$ (follows by scaling symmetry). Thus, it is clear that $w := \boldsymbol{\mathrm{W}}\bp$ has all its entries bounded in absolute value by 1.

We could write $\bp = \boldsymbol{\mathrm{W}}^{-1}w $, thus $\inmod{\bp}^2 = \paren{\boldsymbol{\mathrm{W}}^{-1}w}^{\top}\paren{\boldsymbol{\mathrm{W}}^{-1}w} = w^{\top}\paren{\boldsymbol{\mathrm{W}}^{\top}\boldsymbol{\mathrm{W}}}^{-1}w$. Thus, showing the bound for $w^{\top}\paren{\boldsymbol{\mathrm{W}}^{\top}\boldsymbol{\mathrm{W}}}^{-1}w$ where $w \in \bracket{-1,1}^n$ suffices. We note that $\paren{\boldsymbol{\mathrm{W}}^{\top}\boldsymbol{\mathrm{W}}}$ is a symmetric $(n\times n)$ matrix with diagonal entries 1 and non-diagonal entries bounded in absolute value by $\cos{\theta_0}$. 

Using convergence of the Neumann series $\sum_{k=0}^{\infty} \paren{\mathbb{Id}_{\paren{n\times n}} - \paren{\boldsymbol{\mathrm{W}}^{\top}\boldsymbol{\mathrm{W}}}}^k$ as $\paren{\mathbb{Id}_{\paren{n\times n}} - \paren{\boldsymbol{\mathrm{W}}^{\top}\boldsymbol{\mathrm{W}}}}$ is a bounded operator, we have:
\begin{align}
    \left|\paren{\boldsymbol{\mathrm{W}}^{\top}\boldsymbol{\mathrm{W}}}^{-1} - \mathbb{Id}_{\paren{n\times n}}\right|_{\ell_{\infty}} &\le \left|\sum_{k=1}^{\infty} \paren{\mathbb{Id}_{\paren{n\times n}} - \paren{\boldsymbol{\mathrm{W}}^{\top}\boldsymbol{\mathrm{W}}}}^k\right|_{\ell_{\infty}} \label{eqn: neu1}\\
    &\le \sum_{k=1}^{\infty}\left| \paren{\mathbb{Id}_{\paren{n\times n}} - \paren{\boldsymbol{\mathrm{W}}^{\top}\boldsymbol{\mathrm{W}}}}^k\right|_{\ell_{\infty}}\label{eqn: neu2}\\
    &\le \sum_{k=1}^{\infty} n^{k-1}\cos^{k}{\theta_0}\label{eqn: neu3}\\
    &= \frac{\cos{\theta_0}}{1-n\cos{\theta_0}}\label{eqn: neu4}
\end{align}
where $|B|_{\ell_{\infty}}$ refers to the maximum absolute value of any entry of $B$. \eqnref{eqn: neu1} follows using the Neumann series $\paren{\boldsymbol{\mathrm{W}}^{\top}\boldsymbol{\mathrm{W}}}^{-1}$ = $\sum_{k=0}^{\infty} \paren{\mathbb{Id}_{\paren{n\times n}} - \paren{\boldsymbol{\mathrm{W}}^{\top}\boldsymbol{\mathrm{W}}}}^k$. \eqnref{eqn: neu2} is a direct consequence of triangle inequality. Since entries of $\paren{\mathbb{Id}_{\paren{n\times n}} - \paren{\boldsymbol{\mathrm{W}}^{\top}\boldsymbol{\mathrm{W}}}}^k$ are bounded in absolute value by $\cos{\theta_0}$ thus \eqnref{eqn: neu3} follows. Using a straight-forward geometric sum we get an upper bound on the maximum absolute value of any entry in $\paren{\boldsymbol{\mathrm{W}}^{\top}\boldsymbol{\mathrm{W}}}^{-1} -\, \mathbb{Id}_{\paren{n\times n}}$ in \eqnref{eqn: neu4}.

Now, we note that
\begin{align}
    w^{\top}\paren{\boldsymbol{\mathrm{W}}^{\top}\boldsymbol{\mathrm{W}}}^{-1}w &= w^{\top}\paren{\paren{\boldsymbol{\mathrm{W}}^{\top}\boldsymbol{\mathrm{W}}}^{-1} -\, \mathbb{Id}_{\paren{n\times n}}}w + w^{\top}\paren{\mathbb{Id}_{\paren{n\times n}}}w\nonumber\\
    &\le \sum_{i,j}w_{ij}\paren{\paren{\boldsymbol{\mathrm{W}}^{\top}\boldsymbol{\mathrm{W}}}^{-1} -\, \mathbb{Id}_{\paren{n\times n}}}_{ij}w_{ij} + \inmod{w}^2\nonumber\\
    &\le \sum_{i,j}\left|w_{ij}\paren{\paren{\boldsymbol{\mathrm{W}}^{\top}\boldsymbol{\mathrm{W}}}^{-1} -\, \mathbb{Id}_{\paren{n\times n}}}_{ij}w_{ij}\right| + \inmod{w}^2\nonumber\\
    &\le \sum_{i,j} \left|\paren{\paren{\boldsymbol{\mathrm{W}}^{\top}\boldsymbol{\mathrm{W}}}^{-1} -\, \mathbb{Id}_{\paren{n\times n}}}_{ij}\right| + n\label{eqn: b1}\\
    &\le \frac{n^2\cos{\theta_0}}{1-n\cos{\theta_0}} + n\label{eqn: b2}\\
    & = \frac{n}{1-n\cos{\theta_0}}\label{eqn: b3}
\end{align}
In \eqnref{eqn: b1} we use $w \in \bracket{-1,1}^n$. \eqnref{eqn: b2} follows using \eqnref{eqn: neu4}. 

Since we have $\inmod{\bp}_2^2 = w^{\top}\paren{\boldsymbol{\mathrm{W}}^{\top}\boldsymbol{\mathrm{W}}}^{-1}w$ and $\cos{\theta_0} \le \frac{1}{2n}$, thus using \eqnref{eqn: b3}
$$\inmod{\bp}_2 \le \sqrt{\frac{n}{1-n\cos{\theta_0}}} \le \sqrt{2n}.$$
Scaling the map $\boldsymbol{\mathrm{W}}$ to $\epsilon$ yields the stated claim.

\end{proof}
Under the \assref{assumption: orthogonal} and \assref{assumption: bounded cone}, and bounded norm of $\thetab^*$ and $\hat{\thetab}$, we would establish that $\mathcal{TS}_{\thetab^*}$ is a $d^{\bigO{\log^2 \frac{1}{\epsilon}}}$ size $\epsilon$-approximate teaching set for $\thetab^*$. 
Before establishing the main result, we show the proof of \thmref{thm: boundedclassifier} below. 
Using \eqnref{eqn: largeeps}, we note that:
\begin{align*}
    \forall\,\bx \in \cX\quad &\inmod{\p^{\perp}\Phi(\bx)} \le \sqrt{\epsilon_s} \implies \inmod{\p\Phi(\bx)} \ge \sqrt{1-\epsilon_s}\\
    \forall\,(\bz,y) \in \mathcal{TS}_{\thetab^*}\quad &\inmod{\p^{\perp}\Phi(\bz)} \le \sqrt{\epsilon_s} \implies \inmod{\p\Phi(\bz)} \ge \sqrt{1-\epsilon_s}
\end{align*}
Now, we could further bound the norms of $\p^{\perp}\hat{\thetab}$ and $\p^{\perp}\thetab^*$ using triangle inequality and boundedness of $\curlybracket{\alpha_i}_{i=1}^l$ and $\bracket{\beta_0,\gamma}$ (as shown in \assref{assumption: bounded cone}):
\begin{align*}
    \inmod{\p^{\perp}\thetab^*} = \inmod{\sum_{i=1}^l \alpha_i\cdot \p^{\perp}\Phi(\ba_i)} &\le \sum_{i=1}^l \inmod{\alpha_i\cdot \p^{\perp}\Phi(\ba_i)} \le  \paren{\sum_{i=1}^l \left| \alpha_i\right|}\cdot \sqrt{\epsilon_s} = \boldsymbol{C}_{\epsilon}\cdot \sqrt{\epsilon_s}\\ 
    \inmod{\p^{\perp}\hat{\thetab}} = \inmod{\beta_0\cdot \p^{\perp}\Phi(\ba) + \sum_{j=1}^{r-1}\gamma_j\cdot \p^{\perp}\Phi(\bz_j)} &\le \inmod{\beta_0\cdot \p^{\perp}\Phi(\ba)} + \sum_{i=1}^{r-1} \inmod{\gamma_i\cdot \p^{\perp}\Phi(\bz_i)} \le  \paren{\left|\beta_0\right| + \sum_{i=1}^{r-1} \left| \gamma_i\right|}\cdot \sqrt{\epsilon_s} = \boldsymbol{D}_{\epsilon}\cdot \sqrt{\epsilon_s}
\end{align*}

\begin{proof}[Proof of \thmref{thm: boundedclassifier}]
In the following, we would bound $\left|f^*(\bx) - \hat{f}(\bx)\right|$ by $
\sqrt{\epsilon}$.
In order to bound the modulus, we would split the difference using $\p$ and $\p^{\bot}$ and then analyze the terms correspondingly. We can write any classifier $f$ as $f(\bx) = \thetab\cdot\Phi(\bx) = \p\thetab\cdot\p\Phi(\bx) + \p^{\bot}\thetab\cdot\p^{\bot}\Phi(\bx)$. Thus, we have: 
\begin{align}
    \left|f^*(\bx) - \hat{f}(\bx)\right| &= \left|\p\thetab^*\cdot\mathbb{P}\Phi(\bx) + \p^{\bot}\thetab^*\cdot\mathbb{P}^{\bot}\Phi(\bx) - \p\hat{\thetab}\cdot\mathbb{P}\Phi(\bx) - \p^{\bot}\hat{\thetab}\cdot\mathbb{P}^{\bot}\Phi(\bx)\right|\nonumber\\
    &\le \left|\p\thetab^*\cdot\mathbb{P}\Phi(\bx)- \p\hat{\thetab}\cdot\mathbb{P}\Phi(\bx)\right| + \left|\p^{\bot}\thetab^*\cdot\mathbb{P}^{\bot}\Phi(\bx)-\p^{\bot}\hat{\thetab}\cdot\mathbb{P}^{\bot}\Phi(\bx)\right|\label{eqn: split1}\\
    &\le \left|\p\thetab^*\cdot\mathbb{P}\Phi(\bx)-\p\hat{\thetab}\cdot\mathbb{P}\Phi(\bx)\right| + \inmod{\p^{\bot}\thetab^*-\p^{\bot}\hat{\thetab}}\cdot\inmod{\mathbb{P}^{\bot}\Phi(\bx)}\label{eqn: split2}\\
    &\le \underbrace{\left|\p\thetab^*\cdot\mathbb{P}\Phi(\bx)- \p\hat{\thetab}\cdot\mathbb{P}\Phi(\bx)\right|}_{\bigstar} +\, 
    \paren{\boldsymbol{C}_{\epsilon}+\boldsymbol{D}_{\epsilon}}\cdot \epsilon_s \label{eqn: split3}
\end{align}
\eqnref{eqn: split1} follows using triangle inequality. We can further bound $\left|\p^{\bot}\thetab^*\cdot\mathbb{P}^{\bot}\Phi(\bx)-\p^{\bot}\hat{\thetab}\cdot\mathbb{P}^{\bot}\Phi(\bx)\right|$ using Cauchy-Schwarz inequality and thus \eqnref{eqn: split2} follows. Using the observations: $||\mathbb{P}^{\bot}\hat{\thetab}|| \le \boldsymbol{D}_{\epsilon}\cdot\sqrt{\epsilon_s}$ and $\inmod{\p^{\bot}\thetab^*} \le \boldsymbol{C}_{\epsilon}\cdot\sqrt{\epsilon_s}$,
 we could upper bound $\inmod{\p^{\bot}\thetab^*-\p^{\bot}\hat{\thetab}}$ by $\paren{\boldsymbol{C}_{\epsilon}+\boldsymbol{D}_{\epsilon}}\cdot \sqrt{\epsilon_s}$. Since $\bx \in \cX$ thus $\inmod{\p^{\bot}\Phi(\bx)} \le \sqrt{\epsilon_s}$ (as shown in \eqnref{eqn: largeeps}), which gives \eqnref{eqn: split3}. 



Now, the key is to bound the $(\bigstar)$ appropriately and then the result would be proven. We would rewrite $\p\hat{\thetab}\cdot\p\Phi(\bx)$ in terms of the basis formed by $\{\p\Phi(\bz_i)\}_{i=1}^{r-1} \cup \{\p\theta^*\}$ (by \assref{assumption: orthogonal} $\{\p\Phi(\bz_i)\}_{i=1}^l$ are linearly independent and orthogonal to $\p\theta^*$). Using the basis, we can write $\mathbb{P}\hat{\thetab} = \sum_{i=1}^{r-1}c_i\cdot \mathbb{P}\Phi(\bz_i) + \lambda_r\cdot\mathbb{P}\thetab^*$ for some scalars $c_1,c_2,\cdots,\lambda_r$. Alternatively, we could rewrite $\mathbb{P}\hat{\thetab} = \beta_0\cdot \p\Phi(\ba) + \sum_{j=1}^{r-1}\gamma_j\cdot \p\Phi(\bz_j)$ where $\beta_0 > 0$ (as shown in \appref{appendixsub: solutionexists}). This could be used to note that $\lambda_r > 0$ because $\p\thetab^*\cdot\p\Phi(\ba) > 0$ (cf \secref{subsec.gaussiankernel}).

We study the decomposition of $\p\hat{\thetab}$ in terms of the basis in order to understand the component of $\p\hat{\thetab}$ along $\p\thetab^*$. We observe that: 
\begin{equation}
    \inmod{\p\hat{\thetab}}^2 = \inmod{\sum_{i=1}^{r-1}c_i\cdot \mathbb{P}\Phi(\bz_i)}^2 + \inmod{\lambda_r\cdot\mathbb{P}\thetab^*}^2 \label{eqn:perpnorm}
\end{equation}
Since $\hat{\thetab}$ is a solution to \eqnref{eqn: bounded}, $\hat{\thetab}\cdot\Phi(\bz_i) = 0$ for any $i \in \bracket{r-1}$. Now, we can write the equation in terms of projections as:
\begin{equation}
 \forall i\quad    \mathbb{P}\hat{\thetab}\cdot \mathbb{P}\Phi(\bz_i) + \mathbb{P}^{\bot}\hat{\thetab}\cdot \mathbb{P}^{\bot}\Phi(\bz_i) = 0\label{eqn: expandeqn}
\end{equation}
Using Cauchy-Schwarz inequality on the product $|\mathbb{P}^{\bot}\hat{\thetab}\cdot \mathbb{P}^{\bot}\Phi(\bz_i)|$ we obtain:
\begin{equation}
    |\mathbb{P}^{\bot}\hat{\thetab}\cdot \mathbb{P}^{\bot}\Phi(\bz_i)| \le ||\mathbb{P}^{\bot}\hat{\thetab}||\cdot ||\mathbb{P}^{\bot}\Phi(\bz_i)|| \le \boldsymbol{D}_{\epsilon}\cdot\epsilon_s \nonumber
\end{equation}
Plugging this into \eqnref{eqn: expandeqn}, we get the following bound on $|\mathbb{P}\hat{\thetab}\cdot \mathbb{P}\Phi(\bz_i)|$:
\begin{equation}
  |\mathbb{P}\hat{\thetab}\cdot \mathbb{P}\Phi(\bz_i)| \le \boldsymbol{D}_{\epsilon}\cdot\epsilon_s  \label{eqn: boundproj}
\end{equation}
We denote $V_{O} := \sum_{i=1}^{r-1}c_i\cdot \mathbb{P}\Phi(\bz_i)$. Notice that $V_{O}$ is the orthogonal projection of $\mathbb{P}\hat{\thetab}$ along the subspace $\textbf{span}\langle\p\Phi(\bz_1),\cdots,\p\Phi(\bz_{r-1})\rangle$. Thus, we could rewrite \eqnref{eqn: boundproj} further as:
\begin{equation}
    |\mathbb{P}\hat{\thetab}\cdot \mathbb{P}\Phi(\bz_i)| = |\paren{V_{O} + \lambda_r\cdot\mathbb{P}\thetab^*}\cdot \mathbb{P}\Phi(\bz_i)| = |V_{O}\cdot \mathbb{P}\Phi(\bz_i)| \le \boldsymbol{D}_{\epsilon}\cdot\epsilon_s \nonumber
\end{equation}
Notice that $||\mathbb{P}\Phi(\bz_i)|| \ge \sqrt{1-\epsilon_s}$. Hence, component of $V_{O}$ along $\mathbb{P}\Phi(\bz_i)$ is upper bounded by $\frac{\boldsymbol{D}_{\epsilon}\cdot\epsilon_s}{\sqrt{1-\epsilon_s}}$. 
Since $\curlybracket{\p\Phi(\bz_1),\cdots,\p\Phi(\bz_{r-1})}$ satisfy the conditions of \lemref{lemma: maximum norm} (the smoothness condition mentioned in \secref{subsec.gaussiankernel}) thus we could bound the norm of $V_{O}$ as follows: 
\begin{equation}
  ||V_{O}|| \le \sqrt{2(r-1)}\cdot \frac{\boldsymbol{D}_{\epsilon}\cdot\epsilon_s}{\sqrt{1-\epsilon_s}}  \label{eqn: finalbound}
\end{equation}


Using \eqnref{eqn:perpnorm} and \eqnref{eqn: finalbound} we can lower bound the norm of $\lambda_r\cdot\mathbb{P}\thetab^*$ as follows:
\begin{align}
    \inmod{\mathbb{P}\hat{\thetab}}^2 =& \inmod{\sum_{i=1}^{r-1}c_i\cdot \mathbb{P}\Phi(\bz_i)}^2 +\inmod{\lambda_r\cdot\mathbb{P}\thetab^*}^2
    = \inmod{V_{O}}^2 +\inmod{\lambda_r\cdot\mathbb{P}\thetab^*}^2\nonumber\\
    \implies& \inmod{\lambda_r\cdot\mathbb{P}\thetab^*}^2 \ge \paren{1-\boldsymbol{D}_{\epsilon}^2\cdot\epsilon_s} -  2\paren{r-1}\cdot \frac{\boldsymbol{D}_{\epsilon}^2\cdot\epsilon_s^2}{\paren{1-\epsilon_s}}  \ge 1-2\boldsymbol{D}_{\epsilon}^2\cdot\epsilon_s \label{eqn: actualnorm}
\end{align}
This follows because $\inmod{\mathbb{P}\hat{\thetab}}^2 \ge \paren{1-\boldsymbol{D}_{\epsilon}^2\cdot\epsilon_s} $ as $\inmod{\hat{\thetab}} = \bigO{1}$ and $\sqrt{2(r-1)}\cdot \epsilon_s \le \sqrt{2(r-1)}\cdot \frac{\epsilon}{(\sqrt{d})^s} \le \epsilon$.
With these observations we can rewrite $\paren{\bigstar}$ as follows:
\begin{align}
\centering
    \left|\p\thetab^*\cdot\mathbb{P}\Phi(\bx)- \p\hat{\thetab}\cdot\mathbb{P}\Phi(\bx)\right| &= \left|\p\thetab^*\cdot\mathbb{P}\Phi(\bx)-  \sum_{i=1}^{r-1}c_i\cdot \mathbb{P}\Phi(\bz_i)\cdot\mathbb{P}\Phi(\bx) - \lambda_r\cdot\mathbb{P}\thetab^*\cdot\mathbb{P}\Phi(\bx)\right|\nonumber\\
    &\le \left|\p\thetab^*\cdot\mathbb{P}\Phi(\bx)-  \lambda_r\cdot\mathbb{P}\thetab^*\cdot\mathbb{P}\Phi(\bx)\right| + \left| \sum_{i=1}^{r-1}c_i\cdot \mathbb{P}\Phi(\bz_i)\cdot\mathbb{P}\Phi(\bx)\right|\label{eqn: triangle}\\
    &\le \sqrt{|\boldsymbol{C}_{\epsilon}^2 - 2\boldsymbol{D}_{\epsilon}^2|}\cdot\sqrt{\epsilon_s} + \inmod{\sum_{i=1}^{r-1}c_i\cdot \mathbb{P}\Phi(\bz_i)}\cdot \inmod{\mathbb{P}\Phi(\bx)}\label{eqn: boundfinal1}\\
    &\le \sqrt{|\boldsymbol{C}_{\epsilon}^2 - 2\boldsymbol{D}_{\epsilon}^2|}\cdot\sqrt{\epsilon_s} +\sqrt{2(r-1)}\cdot \frac{\boldsymbol{D}_{\epsilon}\cdot\epsilon_s}{\sqrt{1-\epsilon_s}}\label{eqn: boundfinal2}\\
    & \le \frac{\sqrt{|\boldsymbol{C}_{\epsilon}^2 - 2\boldsymbol{D}_{\epsilon}^2|}\cdot\sqrt{\epsilon}}{(\sqrt{d})^{s/2}} +  2\boldsymbol{D}_{\epsilon}\cdot\epsilon\label{eqn: boundfinal3}\\
    & \le 2\max\left\{\frac{\sqrt{|\boldsymbol{C}_{\epsilon}^2 - 2\boldsymbol{D}_{\epsilon}^2|}}{(\sqrt{d})^{s/2}},\, 2\boldsymbol{D}_{\epsilon}\cdot\sqrt{\epsilon}\right\}\cdot \sqrt{\epsilon}\label{eqn: boundfinal4}
\end{align}
\eqnref{eqn: triangle} is a direct implication of triangle inequality. In \eqnref{eqn: boundfinal1}, in the first term we note that $\lambda_r > 0$ () $\inmod{\p\thetab^*} \ge \sqrt{1-\boldsymbol{C}_{\epsilon}^2\cdot\epsilon_s}$ and use \eqnref{eqn: actualnorm}, and in the second use Cauchy-Schwarz inequality.
\eqnref{eqn: boundfinal2} follows using \eqnref{eqn: finalbound} and that
$||\mathbb{P}\Phi(\bx)||$ is bounded by 1. We could unfold the value of $\epsilon_s \le \frac{\epsilon}{(\sqrt{d})^s}$. This gives us \eqnref{eqn: boundfinal3}. We could rewrite \eqnref{eqn: boundfinal3} to get a bound in terms of $\sqrt{\epsilon}$ to obtain \eqnref{eqn: boundfinal4}.

Now, using \eqnref{eqn: split3} and \eqnref{eqn: boundfinal4}, can bound $\left|f^*(\bx) - \hat{f}(\bx)\right|$ as follows:
\begin{align*}
    \left|f^*(\bx) - \hat{f}(\bx)\right| &\le 2\max\left\{\frac{\sqrt{|\boldsymbol{C}_{\epsilon}^2 - 2\boldsymbol{D}_{\epsilon}^2|}}{(\sqrt{d})^{s/2}},\, 2\boldsymbol{D}_{\epsilon}\cdot\sqrt{\epsilon}\right\}\cdot \sqrt{\epsilon} +  \paren{\boldsymbol{C}_{\epsilon}+\boldsymbol{D}_{\epsilon}}\cdot \frac{\epsilon}{(\sqrt{d})^s}\\
    &\le 3\max\left\{\frac{\sqrt{|\boldsymbol{C}_{\epsilon}^2 - 2\boldsymbol{D}_{\epsilon}^2|}}{(\sqrt{d})^{s/2}},\, 2\boldsymbol{D}_{\epsilon}\cdot\sqrt{\epsilon},\, \paren{\boldsymbol{C}_{\epsilon}+\boldsymbol{D}_{\epsilon}}\cdot \frac{\sqrt{\epsilon}}{(\sqrt{d})^s} \right \}\cdot \sqrt{\epsilon}\\
    &\le 3\boldsymbol{C}'\cdot \sqrt{\epsilon}
\end{align*}
where $\boldsymbol{C}' := \max\left\{\frac{\sqrt{|\boldsymbol{C}_{\epsilon}^2 - 2\boldsymbol{D}_{\epsilon}^2|}}{(\sqrt{d})^{s/2}},\, 2\boldsymbol{D}_{\epsilon}\cdot\sqrt{\epsilon},\, \paren{\boldsymbol{C}_{\epsilon}+\boldsymbol{D}_{\epsilon}}\cdot \frac{\sqrt{\epsilon}}{(\sqrt{d})^s} \right \}$. \\
Notice that all the terms in $\max\left\{\frac{\sqrt{|\boldsymbol{C}_{\epsilon}^2 - 2\boldsymbol{D}_{\epsilon}^2|}}{(\sqrt{d})^{s/2}},\, 2\boldsymbol{D}_{\epsilon}\cdot\sqrt{\epsilon},\, \paren{\boldsymbol{C}_{\epsilon}+\boldsymbol{D}_{\epsilon}}\cdot \frac{\sqrt{\epsilon}}{(\sqrt{d})^s} \right \}$ are smaller than 1 because of boundedness of $\boldsymbol{C}_{\epsilon}$ and $\boldsymbol{D}_{\epsilon}$.
Thus, we have shown a $3C'\cdot\sqrt{\epsilon}$ (where $C'$ is a constant smaller than 1) bound on the point-wise difference of $\hat{f}$ and $f^*$. Now, if we scale the $\epsilon$ and solve for $\epsilon^2/3$, we get the desired bound. Hence, the main claim of \thmref{thm: boundedclassifier} is proven i.e. $ \left|f^*(\bx) - \hat{f}(\bx)\right| \le \epsilon$.
\end{proof}
Now, we would complete the proof of the main result of \secref{subsec.gaussiankernel} which bounds the error incurred by the solution $\hat{\thetab} \in \cA_{opt}(\mathcal{TS}_{\thetab^*})$ i.e. \thmref{thm: gaussian_main_thm}. The point-wise closeness of $f^*$ and $\hat{f}$ established in \thmref{thm: boundedclassifier} would be key in bounding the error. We complete the proof as follows:

\begin{proof}[Proof of \thmref{thm: gaussian_main_thm}]
We show the error analysis when data-points are sampled from the data distribution $\cP$. 
\begin{align}
    \left|\textbf{err}(f^*) - \textbf{err}(\hat{f})\right| &=
    \left|\expover{(\bx,y) \sim \cP}{\max(-y\cdot f^*(\bx), 0)}  - \expover{(\bx,y) \sim \cP}{\max(-y\cdot \hat{f}(\bx), 0)}\right|\label{eqn:final1}\\
    &= \left|\expover{(\bx,y) \sim \cP}{\max(-y\cdot f^*(\bx), 0) - \max(-y\cdot \hat{f}(\bx), 0)}\right|\label{eqn:final2}\\
    &\le \expover{(\bx,y) \sim \cP}{\left|f^*(\bx) - \hat{f}(\bx)\right|}\label{eqn:final3}\\
 &\le \epsilon \label{eqn:final4}
\end{align}
\eqnref{eqn:final1} follows using the definition of $\textbf{err}(\cdot)$ function. Because of linearity of expectation, we get \eqnref{eqn:final2}. In \eqnref{eqn:final3}, we use the observation that modulus of an expectation is bounded by the expectation of the modulus of the random variable $f^*(\bx) - \hat{f}(\bx)$. In \thmref{thm: boundedclassifier}, we showed that for any $\bx \in \cX$, $\left|f^*(\bx) - \hat{f}(\bx)\right| \le \epsilon$. Thus, the main claim follows.
\end{proof}
\newpage

\section{Experimental Evaluation}\label{appendix: experimentals}
In this section, we provide an algorithmic procedure for constructing the $\epsilon$-approximate teaching set, and quantitatively evaluate our theoretical results as presented in \thmref{thm: boundedclassifier} and \thmref{thm: gaussian_main_thm} (cf. \secref{subsec: bounded_error}). 

Results in this section are supplementary to \figref{fig:example-RBF}. 
For a qualitative evaluation of the \emph{$\epsilon$-approximated teaching set}, please refer to 
\figref{fig:Learned_RBF_45c}, which illustrates the learner's Gaussian kernel perceptron learned from the {$\epsilon$-approximated teaching sets} on different classification tasks.

\subsection{Experimental Setup}

Our experiments are carried out on 4 different datasets: the two-moon dataset (2 interleaving half-circles with noise), the two-circles dataset (a large circle containing a small circle, with noise) from sklearn\footnote{https://scikit-learn.org/stable/modules/classes.html\#module-sklearn.datasets}, the Banana dataset\footnote{https://www.scilab.org/tutorials/machine-learning-–-classification-svm} where the two classes are not perfectly separable, and the Iris dataset\footnote{https://archive.ics.uci.edu/ml/datasets/iris} with one of the three classes removed.  
For each dataset, the following steps are performed:

\begin{enumerate}
\item\label{enum: e1} For a given set of data, we, assuming the role of the teacher, find the optimal Gaussian (with $\sigma = 0.9$) separator $\thetab^*$ and plot the corresponding boundaries. We estimate the perceptron loss $\textbf{err}(f^*)$ for this separator by summing up the total perceptron loss on the dataset and averaging over the size of the dataset.
\item\label{enum: e2} For some $s$, we use the degree $s$ polynomial approximation of the Gaussian separator to determine the approximate polynomial boundaries and select $r = \binom{2+s}{s} - 1$ points on the boundaries such that their images in the polynomial feature space are linearly independent. We make a copy of these points and assign positive labels to one copy and negative labels to the other. In addition, we pick 2 more points arbitrarily, one on each side of the boundaries (i.e. with opposite labels). Thus $\mathcal{TS_{\theta^*}}$ of size $2r$ is constructed.
\item\label{enum: e3} Following \assref{assumption: orthogonal}-\ref{assumption: bounded cone}, the Gaussian kernel perceptron learner (with the same $\sigma$ parameter value 0.9) uses only $\mathcal{TS_{\theta^*}}$ to learn a separator $\hat{\thetab}$. The perceptron loss $\textbf{err}(\hat{f})$ w.r.t. the original dataset is calculated by averaging the total perceptron loss over the number of points in the dataset.
\item\label{enum: e4} Repeat Step 2 and Step 3 for $s$ = $2, 3,\cdots, 12$ and record the perceptron loss (i.e the $\max$ error function as shown in \secref{sec.statement}) for the corresponding teaching set sizes $2r$ (where $r$ = $\binom{2+s} {s} - 1$). Then we plot the error $\left|\textbf{err}(f^*) - \textbf{err}(\hat{f})\right|$  as a function of the teaching set size.
\end{enumerate}

The corresponding plots for Steps 1-4 are shown in columns (a)-(d) of \figref{fig:panel-RBF}, where for Step 2 (column (b)), the plots all correspond to when $s=5$.

\subsection{Implementation Details}
In this subsection we provide more algorithmic and numeric details about the implementation of the experiments.

First we describe how the first $r-1$ points are generated in \stepref{enum: e2}. Given that the approximate polynomial separator has been found using the kernel and feature map approximation described in \eqnref{eqn: eqn17} and \eqnref{eqn:eqn18}, we are able to plot the corresponding boundaries, and by the same reasoning as in the case of teaching set generation for the polynomial learner, we need to locate points on the boundaries such that their images in the r-dimensional feature space are linearly independent. We achieve this by sampling points on the zero-contour line and row-reducing the matrix formed by the image of all such points. This way, $r-1$ qualified points can be efficiently located. In addition, as discussed in \secref{subsec: bounded_error}, the teaching points are selected within the radius of some small constant multiple of $\sqrt{R}$ consistently across the experiments. In this case, we have arbitrarily picked the constant to be 4. 

In \stepref{enum: e3}, when the learner learns the separator, we need to ensure \assref{assumption: orthogonal}-\ref{assumption: bounded cone} are satisfied. This is made possible by adding the corresponding constraints to the learner's optimization procedure. Specifically, we need to enforce that 1) the norm of $\hat{\thetab}$ is not far from 1, and 2) $\beta$\footnote{We pick two points outside the orthogonal complement of $\p\thetab^*$, one with positive label and another with negative label. Thus, in place of $\beta_0$ (as used in \secref{subsec.gaussiankernel}) we use $\beta \in \reals^2$ here.} and $\gamma$ are bounded absolutely as mentioned in \assref{assumption: bounded cone}. This is achieved by adjusting the specified bound higher or lower as the current-iteration $\hat{\thetab}$ norm varies during the optimization procedure. Eventually, we normalize $\hat{\thetab}$ and check that the final $\beta$ and $\gamma$ are indeed bounded (i.e. \assref{assumption: bounded cone} is satisfied). 

Finally, the perceptron loss calculated for each value of $s$ is based on 5 separate runs of \stepref{enum: e2}, while for each run, the learner's kernelized Gaussian perceptron learning algorithm is repeated 5 times. The learner's perceptron loss is then averaged over the 25 epochs of the algorithm to prevent numerical inaccuracies that may arise during the learner's constrained optimization process and possibly the teaching set generation process.

\subsection{Results}
We present the experiment results in \figref{fig:panel-RBF}. In the right-most plot of each row, the estimates of $\left|\textbf{err}(f^*) - \textbf{err}(\hat{f})\right|$ are plotted against the teaching set sizes $2r$ corresponding to $s=2,\cdots,12$ (as discussed in \secref{subsec: gaussian_kernel_approx}). As can be observed from the shape of the curves in plots of column (d), indeed, our experimental results confirm that the number of teaching examples needed for $\epsilon$-approximate teaching is upper-bounded by $d^{\bigO{\log^2 \frac{1}{\epsilon}}}$ for Gaussian kernel perceptrons.

\begin{figure*}[t]
\centering
	\begin{subfigure}[b]{0.24\textwidth}
	   \centering
		\includegraphics[width=\linewidth]{fig/supp_fig1-1b.png}
		\label{fig:Teacher_RBf_41a}
	\end{subfigure}
	\begin{subfigure}[b]{0.24\textwidth}
	    \centering
		\includegraphics[width=\linewidth]{fig/supp_fig1-2b.png}
		\label{fig:Polynomial_approx_41b}
	\end{subfigure}	
	\begin{subfigure}[b]{0.24\textwidth}
	    \centering
		\includegraphics[width=\linewidth]{fig/supp_fig1-3b.png}
		\label{fig:Learned_RBF_41c}
	\end{subfigure}	
	\begin{subfigure}[b]{0.24\textwidth}
	    \centering
		\includegraphics[width=\linewidth]{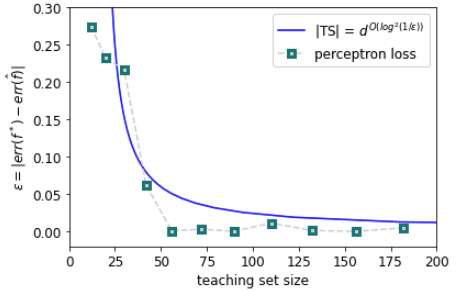}
		\label{fig:Epsilon_size_41d}
	\end{subfigure}	
	\\
	\begin{subfigure}[b]{0.24\textwidth}
	   \centering
		\includegraphics[width=\linewidth]{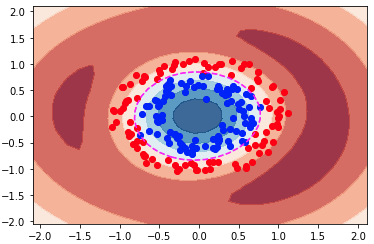}
		\label{fig:Teacher_RBf_42a}
	\end{subfigure}
	\begin{subfigure}[b]{0.24\textwidth}
	    \centering
		\includegraphics[width=\linewidth]{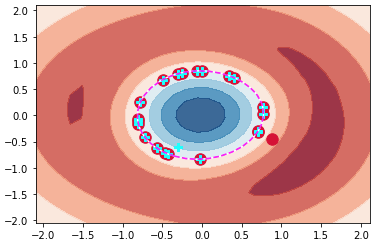}
		\label{fig:Polynomial_approx_42b}
	\end{subfigure}	
	\begin{subfigure}[b]{0.24\textwidth}
	    \centering
		\includegraphics[width=\linewidth]{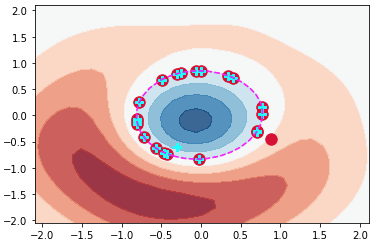}
		\label{fig:Learned_RBF_42c}
	\end{subfigure}	
	\begin{subfigure}[b]{0.24\textwidth}
	    \centering
		\includegraphics[width=\linewidth]{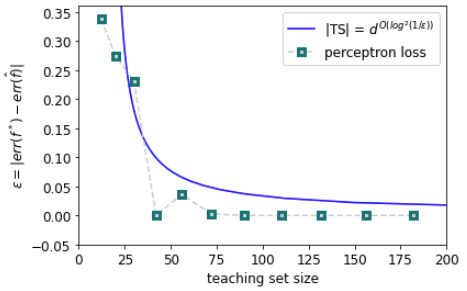}
		\label{fig:Epsilon_size_42d}
	\end{subfigure}	
	\\
	\begin{subfigure}[b]{0.24\textwidth}
	   \centering
		\includegraphics[width=\linewidth]{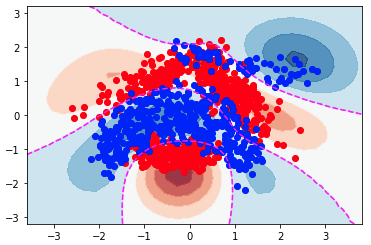}
		\label{fig:Teacher_RBf_43a}
	\end{subfigure}
	\begin{subfigure}[b]{0.24\textwidth}
	    \centering
		\includegraphics[width=\linewidth]{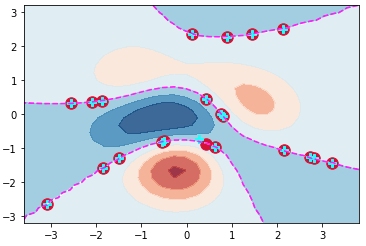}
		\label{fig:Polynomial_approx_43b}
	\end{subfigure}	
	\begin{subfigure}[b]{0.24\textwidth}
	    \centering
		\includegraphics[width=\linewidth]{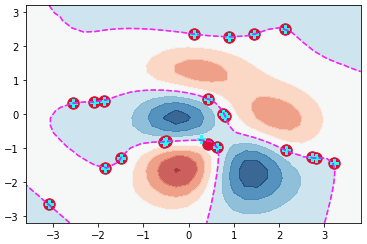}
		\label{fig:Learned_RBF_43c}
	\end{subfigure}	
	\begin{subfigure}[b]{0.24\textwidth}
	    \centering
		\includegraphics[width=\linewidth]{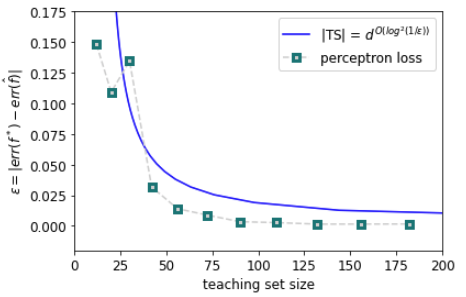}
		\label{fig:Epsilon_size_43d}
	\end{subfigure}	
	\\
	\begin{subfigure}[b]{0.24\textwidth}
	   \centering
		\includegraphics[width=\linewidth]{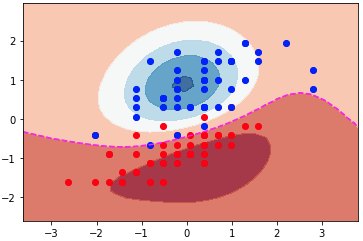}
		\caption{}
		\label{fig:Teacher_RBf_45a}
	\end{subfigure}
	\begin{subfigure}[b]{0.24\textwidth}
	    \centering
		\includegraphics[width=\linewidth]{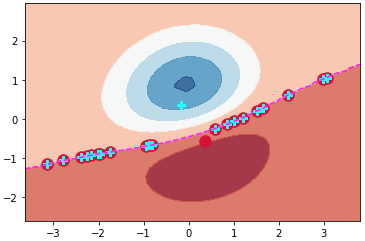}
		\caption{}
		\label{fig:Polynomial_approx_45b}
	\end{subfigure}	
	\begin{subfigure}[b]{0.24\textwidth}
	    \centering
		\includegraphics[width=\linewidth]{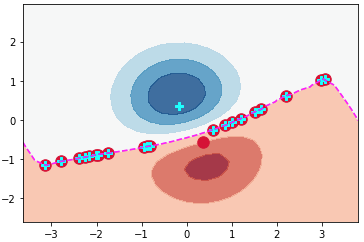}
		\caption{}
		\label{fig:Learned_RBF_45c}
	\end{subfigure}	
	\begin{subfigure}[b]{0.24\textwidth}
	    \centering
		\includegraphics[width=\linewidth]{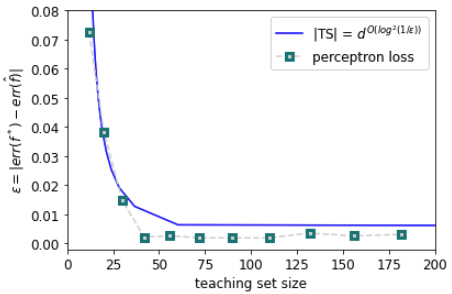}
		\caption{}
		\label{fig:Epsilon_size_45d}
	\end{subfigure}	
  \caption{Constructing the $\epsilon$-approximate teaching set $\mathcal{TS}_{\thetab^*}$ for a Gaussian kernel perceptron learner. Each row corresponds to the numerical results on a different dataset as described in the beginning of \appref{appendix: experimentals}.
  For each row from left to right: (a) optimal Gaussian boundary and the data set; (b) Teacher identifies a (degree-5) polynomial approximation of the Gaussian decision boundary and finds the $\epsilon$-approximate teaching set $\mathcal{TS}_{\thetab^*}$ (marked by cyan plus markers and red dots); (c) Learner learns a Gaussian kernel perceptron from the optimal teaching set in the previous plot; (d) $\abs{\mathcal{TS}}\text{-}\epsilon$ plot for degree-2 to degree-12 polynomial approximation teaching results. The blue curve corresponds to $d^{\bigO{\log^2 \frac{1}{\epsilon}}}$ = $2^{\bigO{\log^2 \frac{1}{\epsilon}}}$ where $d=2$.}
	\label{fig:panel-RBF}
\end{figure*}

}
}
{}
\end{document}